\documentclass[hidelinks,onefignum,onetabnum]{siamart220329}

\usepackage{lipsum}
\usepackage{amsfonts}
\usepackage{graphicx}
\usepackage{epstopdf}
\usepackage{algorithmic}
\ifpdf
  \DeclareGraphicsExtensions{.eps,.pdf,.png,.jpg}
\else
  \DeclareGraphicsExtensions{.eps}
\fi

\title{Computable Lipschitz bounds for deep neural networks}

\author{Moreno Pintore\thanks{MEGAVOLT team, INRIA, F-75013 Paris, France -- SCAI, Sorbonne Universit\'e, F-75005 Paris, France. (moreno.pintore@inria.fr).}
\and Bruno Despr\'es\thanks{Sorbonne Universit\'e, Universit\'e Paris Cit\'e, CNRS, INRIA, Laboratoire Jacques-Louis Lions, LJLL, F-75005 Paris, France,
  (bruno.despres@sorbonne-universite.fr).}
}

\usepackage{amsopn}

\input{style}

\ifpdf
\hypersetup{
  pdftitle={Certified computable Lipschitz bounds for Deep Neural Networks},
  pdfauthor={Moreno Pintore and Bruno Despr\'es}
}
\fi

\usepackage{hyperref}
\usepackage{multirow}

\begin{document}

\maketitle

\begin{abstract}
Deriving sharp and computable upper bounds of the Lipschitz constant of deep neural networks is crucial  to formally guarantee the robustness of neural-network based
 models. We analyse three existing upper bounds written for the $l^2$ norm. We  highlight the importance of working with  the $l^1$ and $l^\infty$ norms and we propose two novel bounds  for both feed-forward fully-connected neural networks and convolutional neural networks. We treat the technical difficulties related to 
 convolutional neural networks with two different methods, called explicit and implicit. Several numerical tests empirically confirm the theoretical results, help to quantify  the relationship between the presented bounds
and establish the better accuracy of the new bounds. Four numerical tests  are studied: two where the output is derived from an
analytical  closed form are  proposed;  another one  with random matrices; and the last one for convolutional neural networks trained on the MNIST dataset.
We observe that one of our bound is  optimal in the sense that it is exact for the first  test with the simplest analytical form and it is better than other bounds for the other tests.
\end{abstract}






\section{Introduction}

Deep neural networks are more and more common in most scientific applications, even though they are still unstable in presence of specific small input perturbations. The most common examples of such instabilities are the adversarial attacks in the context of image classification  \cite{madry2018towards}. 
We also remark that, for engineering applications, the stability of the deep neural networks approximating physical functions with respect to small perturbations has to be controlled and quantified. A recent study in this direction is \cite{bouchi}.
That is why it is of supreme importance for the development of the discipline to establish methods which  qualify and quantify  the stability of functions or operators represented by deep neural networks. The interest in the stability of the neural networks is also partially inspired by some of the previous works of the manuscript authours, where the theory of deep neural networks is analyzed from different perspectives (see \cite{berrone2022solving, berrone2022variational, desp:book}).

In \cite{szegedy2013intriguing},  Szegedy-Goodfellow-et-al analyse the relationship between the Lipschitz constant of deep neural networks and their stability properties \cite{franceschi2018robustness}, proposing an upper bound obtained by multiplying the Lipschitz constants of all the layers. This  bound, denoted by $K_*$ in the following discussion, is very pessimistic 
and so it is meaningless when the number of layers grows (the deep regime). Nevertheless, $K_*$ has been extensively utilized during the training or in the architecture of deep neural networks to improve the stability of the model \cite{couellan2021coupling,liu2022learning,gouk2021regularisation}. Sharper estimates have been obtained in  Combettes-Pesquet \cite{combettes2020lipschitz} and in Scaman-Virmaux \cite{virmaux2019lipschitz}. In the current manuscript, we are interested in further extending these two specific works  by presenting new certified Lipschitz bounds in $l^p$  norms ($1\leq p \leq \infty$) and their extensions to convolutional  neural networks. We are also interested in  measuring and  comparing 
the efficiency of all these estimates on simple tests.
 
Even though it is not the subject of the current paper, we highlight that there exist some estimates of the Lipschitz constant based on semidefinite minimization problems \cite{fazlyab2019efficient} or on polynomial approximation approaches \cite{latorre2020lipschitz,chen2020semialgebraic} (see \cite{gupta2022multivariate} for a compact overview of both approaches), as well as estimates of the local Lipschitz constant \cite{weng2018towards,zhang2018efficient,lee2020lipschitz}. We also highlight that all these bounds can be employed in theoretical analyses such as \cite{katselis2021concentration,huang2019stochastic} to better characterize the neural network output.

In our opinion, our  main results are as follows.
\begin{itemize}
\item We provide theoretical reasons why the $l^1$ norm and the $l^\infty$ norm are in some cases preferable in evaluating
the Lipschitz bounds, rather  than  $l^2$-based norms which are more standard.
\item We define new bounds $K_3$ and $K_4$. The bound $K_4$  is sharper than the $K_1$ constant (which comes form the Combettes-Pesquet work \cite{combettes2020lipschitz})
in the  $l^1$  and  $l^\infty$ norms.
\item We extend in two different ways the discussed bounds to the case of convolutional neural networks.
\item We test our bounds on four numerical tests of different origins which confirm the better accuracy provided by our new bound $K_4$.
In particular, we observe a surprising behavior for our second test problem for which we know the exact value of the Lipschitz constant.
In this case, the numerical bound $K_4$ is exact as shown in Table \ref{tab:lips_for_x2}.
\end{itemize}

The paper is structured as follows. In Section \ref{sec:fcnn}, we focus on fully-connected feed-forward neural networks. In particular, for this kind of neural networks, we summarize and prove existing Lipschitz upper bounds in Section \ref{sec:prev_works_fcnn} and we prove new bounds in Section \ref{sec:new_bounds}. Convolutional neural networks are analysed in Section \ref{sec:conv}, where two approaches to bound the Lipschitz constants of networks including max-pooling layers are discussed in Sections \ref{sec:expl_app} and \ref{sec:impl_app}. Numerical results in strong agreement with the theoretical ones are provided in Section \ref{sec:num_res}, where we consider a feed-forward fully-connected neural network with random weights 
\cite{combettes2020lipschitz}, with weights chosen to efficiently approximate polynomial functions and convolutional neural networks. 
 In Sections \ref{sec:x2} and \ref{sec:xy} we propose two benchmarks where the output of a deep neural network is known in a simple closed form, because we believe that this is crucial in order to improve the theoretical understanding of deep neural networks with an arbitrary number of layers but in simplified scenarios. Section \ref{sec:conclusion} includes conclusion remarks and future perspectives.
\smallskip

{\bf Notation} For  any vector  $x\in \mathbb R^a$ with arbitrary size $a\geq 1$, we will denote by $\Vert x\Vert_{l^p(\mathbb R^a)}=\left(\sum_{ i =1}^a|x_i|^p\right)^{1/p}$ its $l^p$ norm.
Popular norms \cite{hingham} are the $l^1$ norm $\|x\|_{l^1(\mathbb R^a)}=\sum_{ i =1}^a|x_i|$, the $l^2$ norm $\|x\|_{l^2(\mathbb R^a)}^2=\sum_{i=1}^a x_i^2$ and the $l^\infty$ norm $\|x\|_{l^\infty(\mathbb R^a)}= \max_{i=1}^a
|x_i|$.
It yields the induced norm for matrices $M\in \mathcal M_{ab}(\mathbb R)=\mathbb R^{a\times b}$ where the output size $a\geq 1$ and the input size $b\geq 1$ are arbitrary. The induced norm is 
\begin{equation} \label{eq:valval}
\|M\|_{l^p(\mathcal M_{ab}(\mathbb R) )}= \max _{x\neq 0} \frac{ \| Mx\|_{l^p(\mathbb R^a)}}{\|x\|_{l^p(\mathbb R^b)}}.
\end{equation}
It will be denoted by the simpler form 
 $\|M\|=\|M\|_{l^p(\mathcal M_{ab}(\mathbb R) )}$ if there is no ambiguity.

\section{Fully-connected feed-forward neural networks}\label{sec:fcnn}

A generic fully-connected feed-forward neural network function $f$ \cite{goodi,desp:book} is  written under the form  
\begin{equation}\label{eq:fcnn_def}
f = f_\ell \circ S_\ell \circ f_{\ell-1} \circ S_{\ell-1} \circ\dots \circ S_1 \circ f_0.
\end{equation}
 Feed-forward neural networks are mostly used for regression purposes, which means that $f$ is usually close to a smooth objective function \cite{desp:book}.
Here  the functions $f_i$, $i=1,\dots,\ell$,  denote affine functions with varying input and output dimensions $(a_0,a_1, \dots, a_{\ell+1})\in \mathbb N^{\ell+2}$ with $a_i>0$, for   $i=1,\dots,\ell$.
More precisely, $f_i(x_i) = W_ix_i+b_i\in \mathbb R^{a_{i+1}}$ for all $x_i\in \mathbb R^{a_i}$. The parameters of the affine function $f_i$ are  called 
the weights $W_i \in \mathcal M_{a_{i+1}, a_i}(\mathbb R)$ and the biases  $b_i\in \mathbb R^{a_{i+1}}$.
The intermediate functions  $S_i$, $i=1,\dots,\ell$,  are nonlinear activation functions \cite{sharma2017activation} such that $0\le S_i'(x)\le 1$ a.e.  $ x\in\R$. The ReLU function $R(x)=\max(0,x)$ is a  popular activation function
that we will use in our numerical tests. The activation functions $S_i$ are all  Lipschitz by assumption. Note that in practice, that is in the softwares that are used in machine learning, the derivative of all activation functions are defined everywhere. For example a standard choice is $R'(0)=0$ (see \cite{tensorflow2015-whitepaper, pytorch, jax2018github}).
 Any activation function is generalized component wise to  vectors of arbitrary dimension. The index $i$ in $S_i$ explains that one can change the activation functions from one layer to the other.
 The number $a_i$ is referred to as the number of neurons of the $i$-th layer.
The integer $\ell$ is   the number of layers of the neural network function  (\ref{eq:fcnn_def}). 

The regularity of  $f$ is
 $$
 f\in C^0(\mathbb R^{a_0})^{a_{\ell+1}}\bigcap \mbox{Lip}(\mathbb R^{a_0})^{a_{\ell+1}}
 $$ 
 since it is the composition of continuous and globally Lipschitz functions.
 The Rademacher's Theorem states that $f$ is differentiable almost everywhere.
The Lipschitz constant of the function $f$  that we consider is 
$$
L=
\sup_{x\in \mathbb R^{a_0}} \| \nabla f(x)\|_{l^p \left(\mathcal M_{a_{\ell+1},a_0}(\mathbb R)  \right)}
=
\left( \| \nabla f\|_{l^p \left(\mathcal M_{a_{\ell+1},a_0}(\mathbb R) \right)} \right) _{L^\infty(\mathbb R^{a_0})}.
$$ 
Let us denote by $D_r = D_r(x)$ the square diagonal matrix 
$
D_r = S_r'\circ f_{r-1}\circ S_{r-1}\circ\dots\circ S_1\circ f_0(x)$. 
With our assumptions,  the diagonal  coefficients  of $D_r$ are between 0 and 1. We define the set 
\begin{equation} \label{eq:Dr}
\mathcal D_r=\left\{
D_r\in \mathcal M_{a_r,a_r} (\mathbb R)\mid D_r \mbox{ is a diagonal matrix with coefficients between 0 and 1} 
\right\}.
\end{equation}
One can thus write $D_r\in \mathcal D_r$.
The chain rule shows that the gradient of $f$ is the  product of matrices
\begin{equation} \label{eq:gradf}
\nabla f(x) = W_\ell D_\ell (x)W_{\ell-1}D_{\ell-1}(x)\dots D_1 (x)W_0,
\end{equation}
where only the matrices $D_r$, $r=1,\dots,\ell$, depend on $x$.

For convenience, we note $\mathcal D= \mathcal D_\ell \times \mathcal D_{\ell-1} \times \dots \times \mathcal D_1$,
$D=D(x)=\left(D_\ell (x), D_{\ell-1}(x), \dots, D_1 (x) \right)\in \mathcal D$ and $W=(W_\ell, W_{\ell-1},\dots, W_0)$.
As in \cite{szegedy2013intriguing}, we introduce the quantity 
\begin{equation}\label{eq:fcnn_starting_bound}
K = K(W) = \max_{D\in{\cal{D}}} \Vert W_\ell D_\ell W_{\ell-1}D_{\ell-1}\dots D_1 W_0 \Vert ,
\end{equation}
so that one clearly has
$$
L\le K.
$$
\begin{remark}[Complexity] \label{rem:1}
Set  
$
\mbox{Ext} \left( \mathcal D_r \right)=\left\{
D_r\in \mathcal D_r \mid  \mbox{ the diagonal  coefficients are  equal to  0 or  1} 
\right\}
$
which is the finite ensemble of   exterior points of the convex set $ \mathcal D_r$.
In (\ref{eq:fcnn_starting_bound}) one notices that $K$ is a convex function of $D_r$, so the maximal value $K$ is  reached at the external points.
Therefore one also  has 
$$
K = \max_{D\in {\rm Ext} \left( {\cal{D}} \right)} \Vert W_\ell D_\ell W_{\ell-1}D_{\ell-1}\dots D_1 W_0 \Vert 
$$
where $ {\rm Ext} \left( {\cal{D}}\right)= \mbox{Ext} \left( \mathcal D_\ell \right)\times \dots \times \mbox{Ext} \left( \mathcal D_0 \right)$.
Since  $ {\rm Ext}({\mathcal D}) $ is finite, 
it is  possible to calculate $K$ by exhaustion of the possibilities. Nevertheless 
the number of cases is the number of corners of the hypercube $\mathcal D\in \mathbb R^{a_\ell+\dots + a_0}$.
So the complexity of this calculation is  $O\left( 2^{a_\ell+\dots + a_0}\right)$ where $a_\ell+\dots + a_0$ is the total number of neurons.
That is, the complexity of the calculation of $K$ is exponential with respect to the total number of neurons.
In many reasonable cases, this cost makes this direct  calculation of $K$ just impossible.
\end{remark}

Our goal in this work is to examine some bounds on $K$ which are computable and  sharp.
 This is what we call {\bf certified  bounds}.

\subsection{Previous works}\label{sec:prev_works_fcnn}

In this section we introduce some upper bounds for the Lipschitz constant of fully-connected feed-forward neural networks available in literature. 

\begin{definition}[Worst bound]\label{def:worst_bound}
$ K_\star:= \prod_{r=0}^\ell \Vert W_r\Vert$.
\end{definition}

The first bound 
$
 K\leq K_\star
 $
 is evident from the sub-multiplicativity of the norm \cite{combettes2020lipschitz,virmaux2019lipschitz,desp:book}.
However, it will be clear in the numerical Section  that this bound is not sharp and is actually the worst one. 

The next result is an excerpt  from \cite{combettes2020lipschitz}.
For convenience, we begin with a definition.

\begin{definition} \label{defi:wts}
Given $0\le s< t\le\ell+1$, define the matrices
$W_{(t,s)} = W_{t-1} W_{t-2} \dots W_{s+1} W_{s}$. 
\end{definition}

\begin{theorem}[Combettes-Pesquet bound]\label{def:kcp}
One has the bound $K\leq K_1$ where
\begin{equation}\label{eq:def_cp}
 K_{1} = \dfrac{1}{2^\ell} \sum_{1\le r_1< r_2<\dots< r_n\le\ell} \Vert W_{(\ell+1,r_{n})}\Vert \Vert W_{(r_{n},r_{n-1})}\Vert \dots
\Vert W_{(r_2,r_1)}\Vert   \Vert W_{(r_1,0)}\Vert
\end{equation}
\end{theorem}

\begin{remark}
a)  Necessarily $n\le \ell$  in the formula (\ref{eq:def_cp}),  and b) the number of terms in the sum is $2^\ell$.
\end{remark}

\begin{remark}
We prove below in Proposition \ref{prop:k1_sharper_kstar} that  $K_1\leq K_\star$.
But before proving in the inequality, we stress that  there are two intuitive reasons to argue that $K_1$ could be much smaller than $K_\star$, that is $K_1\ll K_\star$,
in some cases.
The first one is that the decomposition for $n=\ell$ gives back one term which is equal to  $ \dfrac{1}{2^\ell} K_\star$. It results in a drastic reduction of the influence of the worst bound in the deep regime.
The second one is that for all $n<\ell$, there is a least one matrix  $W_{(r',r)}$ with $r'\geq r+2$. Then this term   is the product of at least two matrices, for which  which one can expect some cancellations
and thus a strong reduction in the numerical value of the norm of the product.
\end{remark}


\begin{proof}
The technical part of the proposed  proof is a simplification of the one in  \cite{combettes2020lipschitz}. We also show that the bound (\ref{eq:def_cp}) holds for all induced
 norms, while the seminal reference  \cite{combettes2020lipschitz} considers only the $l^2$ norm. 
 
 One expresses the matrices $D_r\in{\cal M}_{a_r,a_r}(\R)$ in \eqref{eq:fcnn_starting_bound} as 
 \begin{equation} \label{eq:dtoz}
 D_r=\dfrac12 ( I_r+Z_r)
 \end{equation}
  where $I_r\in{\cal M}_{a_r,a_r}(\R)$ is the identity matrix and $Z_r=2D_r -I_r$ is such that $\Vert Z_r\Vert\le1$. We denote by ${\cal Z}_r$ the set of matrices ${\cal Z}_r = \{Z_r| \exists D_r\in {\cal D}_r : Z_r=2D_r -I_r\}$. This is generalized as  ${\cal Z} = {\cal Z}_\ell \times\dots\times {\cal Z}_1$ and we  introduce the notation 
 $$
 Z=\left(Z_\ell , \dots, Z_1  \right)\in \mathcal Z.
 $$
 To explain the interest of the matrices $Z_r$, let us consider the following example  with 4 hidden layers
  \begin{equation} \label{eq:gradf1}
 f=f_4\circ S_4 \circ f_3 \circ S_3 \circ f_2 \circ S_2 \circ f_1 \circ S_1 \circ f_0.
 \end{equation}
 The gradient formula (\ref{eq:gradf}) becomes
 \begin{equation} \label{eq:gradf2}
 \nabla f=\frac1{2^4}\   W_4 (I+Z_4)  W_3 (I+Z_2) W_2 (I+Z_2) W_1 (I+Z_1) W_0
 \end{equation}
 where the identity matrices  have different sizes in the general case.
  Expansion of all terms  yields
  \begin{equation} \label{eq:gradf3}
 \nabla f =\frac1{2^4} \ (W_4W_3W_2W_1W_0+\dots + W_4Z_4W_3Z_3W_2Z_2W_1Z_1W_0)
  \end{equation}
 where in the parenthesis,  the first term is the product of all weight matrices, the last term is the product of all weight matrices combined with the $Z_r$,
 and  the middle  term  $\dots$ represents  all other possible chunks of weight matrices $W_r$ separated with matrices  $Z_r$.
 Then a triangular inequality yields the result because $\|Z_r\| =1$ for all $r$.
 
 The general case is treated as follows.
 \begin{equation} \label{eq:k3}
 \begin{aligned}
 K &= \max_{D\in{\cal{D}}} \Vert W_\ell D_\ell W_{\ell-1}D_{\ell-1}\dots D_1 W_0 \Vert \\
 &= \max_{Z\in{\cal{Z}}} \left\Vert W_\ell \left(\dfrac12 ( I_\ell+Z_\ell)\right) W_{\ell-1}\left(\dfrac12 ( I_{\ell-1}+Z_{\ell-1})\right)\dots \left(\dfrac12 ( I_1+Z_1)\right) W_0 \right\Vert\\
 &\leq \dfrac{1}{2^\ell} \max_{Z\in{\cal{Z}}}\sum_{1\le r_1< r_2<\dots< r_n\le\ell} \Vert W_{(\ell+1,r_{n})}Z_{r_n} W_{(r_{n},r_{n-1})} \dots
 W_{(r_2,r_1)}Z_{r_1} W_{(r_1,0)}\Vert\\
 &\le\dfrac{1}{2^\ell} \max_{Z\in{\cal{Z}}}\sum_{1\le r_1< r_2<\dots< r_n\le\ell} \Vert W_{(\ell+1,r_{n})}\Vert \Vert Z_{r_n}\Vert \Vert W_{(r_{n},r_{n-1})}\Vert \dots
\Vert W_{(r_2,r_1)}\Vert  \Vert Z_{r_1}\Vert \Vert W_{(r_1,0)}\Vert\\
&\le \dfrac{1}{2^\ell} \sum_{1\le r_1< r_2<\dots< r_n\le\ell} \Vert W_{(\ell+1,r_{n})}\Vert \Vert W_{(r_{n},r_{n-1})}\Vert \dots
\Vert W_{(r_2,r_1)}\Vert   \Vert W_{(r_1,0)}\Vert 
 = K_1,
\end{aligned}
\end{equation}
where  we have  removed the dependence on $Z$ by using the bound $\Vert Z_r\Vert \le 1$.
\end{proof}

{
\begin{proposition}\label{prop:k1_sharper_kstar}
One has the bound $K_1\le K_*$.
\end{proposition}
\begin{proof}
One has 
\begin{equation}\label{eq:wts_bound}
\Vert W_{(t,s)} \Vert = \Vert  W_{t-1} W_{t-2} \dots W_{s+1} W_{s}\Vert \le \Vert  W_{t-1}\Vert\Vert W_{t-2}\Vert \dots \Vert W_{s+1}\Vert\Vert W_{s}\Vert
\end{equation}
because of the norm submultiplicativity. 
We now bound all the norms involved in the definition of $K_1$ as in \ref{eq:wts_bound} obtaining
 \begin{equation*}
 \begin{aligned}
 K_1 &= \dfrac{1}{2^\ell} \sum_{1\le r_1< r_2<\dots< r_n\le\ell} \Vert W_{(\ell+1,r_{n})}\Vert \Vert W_{(r_{n},r_{n-1})}\Vert \dots
\Vert W_{(r_2,r_1)}\Vert   \Vert W_{(r_1,0)}\Vert \\
& \le  \dfrac{1}{2^\ell} \sum_{1\le r_1< r_2<\dots< r_n\le\ell} \Vert W_\ell\Vert \Vert W_{\ell-1}\Vert \dots \Vert W_1\Vert \Vert W_0\Vert 
=  \dfrac{1}{2^\ell} \sum_{1\le r_1< r_2<\dots< r_n\le\ell} K_*= K_*.
\end{aligned}
\end{equation*}
The last equality holds because the sum includes $2^\ell$ identical terms.
\end{proof}
}

The next result is an extension to generic $l^p$ matrix norms of the Virmaux-Scaman upper bound  \cite{virmaux2019lipschitz}.
One starts from the generic  decomposition of $W_i \in \mathcal M_{a_{i+1}, a_i}(\mathbb R)$ written as
$$
W_i = U_i\Sigma_i V_i^T.
$$ 
This decomposition is a singular value decomposition \cite{brunton2019singular} by choosing 
 $U_i \in \mathcal M_{a_{i+1}, a_{i+1}}(\mathbb R)$ and $V_i \in \mathcal M_{a_{i}, a_i}(\mathbb R)$ as orthogonal matrices, and by imposing that $\Sigma$ is a diagonal matrix with positive entries.
 Other choices are of course possible.

\begin{theorem}[Virmaux-Scaman bound]\label{virmaux2019lipschitz}
One has the bound $K\le K_2$ where
\begin{equation}\label{eq:seqlip_def}
K_2 = \prod_{i=0}^{\ell-1} \max_{D_{i+1}\in {\cal D}_{i+1}} \Vert \widetilde \Sigma_{i+1} V_{i+1}^T D_{i+1} U_i\widetilde \Sigma_i\Vert,
\end{equation}
where $\widetilde \Sigma_0=\Sigma_0 V_0^T$, $\widetilde \Sigma_\ell = U_\ell \Sigma_\ell$ and $ \widetilde \Sigma_i=\Sigma_i^{1/2}$ for all $i\notin\{0,\ell\}$. 

\end{theorem}

\begin{proof}
The theorem can be proved in the following way:
\begin{equation*}
\begin{aligned}
K &= \max_{D\in{\cal{D}}} \Vert \left(U_\ell\Sigma_\ell V_\ell^T\right) D_\ell \left(U_{\ell-1}\Sigma_{\ell-1} V_{\ell-1}^T\right)D_{\ell-1}\dots D_2 \left(U_1\Sigma_1 V_1^T\right) D_1 \left(U_0\Sigma_0 V_0^T\right) \Vert \\
&= \max_{D\in{\cal{D}}} \Vert \widetilde \Sigma_\ell V_\ell^T D_\ell U_{\ell-1}\widetilde \Sigma_{\ell-1}\widetilde \Sigma_{\ell-1} V_{\ell-1}^TD_{\ell-1}\dots D_2 U_1\widetilde\Sigma_1 \widetilde\Sigma_1V_1^T D_1 U_0\widetilde \Sigma_0  \Vert \\
&\le \prod_{i=0}^{\ell-1} \max_{D_{i+1}\in {\cal D}_{i+1}} \Vert \widetilde \Sigma_{i+1} V_{i+1}^T D_{i+1} U_i\widetilde \Sigma_i\Vert.
\end{aligned}
\end{equation*}
\end{proof}

The prescription from \cite{virmaux2019lipschitz} is to take 
 $\widetilde \Sigma_0=\Sigma_0$ and $\widetilde \Sigma_\ell = \Sigma_\ell$, when  the $l^2$ norm is used, because $V_0^T$ and $U_\ell$ are orthogonal matrices. Various similar bounds can be obtained with other matrix factorizations.

\begin{remark}[Complexity of the Virmaux-Scaman bound] \label{rem:scb}
By comparison with Remark \ref{rem:1}, one can calculate the Virmaux-Scaman bound by varying all $D_{i+1}\in \mbox{Ext}\left( {\cal D}_{i+1} \right)$ independently one of the other.
Therefore a direct calculation is possible with a cost
$O\left(2^{a_\ell }+\dots+2^{a_0} \right)$. Clearly, this is much less than the complexity of $K$. However, this is still impracticable if the number of neurons on one layer
is reasonably large (which is the case for standard neural networks). For this reason, in \cite{virmaux2019lipschitz} the authors propose to approximate $K_2$ instead of computing it exactly.
\end{remark}

\subsection{New certified bounds}\label{sec:new_bounds} 

In this section, we introduce two new upper bounds for the Lipschitz constant of fully-connected feed-forward neural networks. 
 Since we will need either one of the  equality expressed in (\ref{eq:norm_equiv_per_abs}),
 the symbol $\Vert\cdot\Vert$ will denote in this Section either $\Vert\cdot\Vert_1$ or $\Vert\cdot\Vert_\infty$ (and only one of these two norms). 

\begin{definition}[Element-wise absolute value of a matrix] \label{defi:df}
Given a matrix $A\in{\cal M}_{m,n}(\R)$ with entries $\{A_{ij}\}_{ij}$, we denote by $A^{\rm{abs}}\in{\cal M}_{m,n}(\R)$ the matrix obtained by applying the absolute value to each entry of $A$, i.e. $(A^{\rm{abs}})_{ij} = |A_{ij}|$, $\forall i,j$. 
\end{definition}

From the well  known  \cite{hingham} identities   $\Vert A\Vert_1 = \max_{1\le j\le n}\sum_{i=1}^m|A_{ij}|$
and $\Vert A\Vert_\infty = \max_{1\le i\le m}\sum_{j=1}^n|A_{ij}|$, one obtains 
 the equalities 
\begin{equation}\label{eq:norm_equiv_per_abs}
\Vert A\Vert_1 = \Vert A^{\rm{abs}}\Vert_1 \mbox{ and }
\Vert A\Vert_\infty = \Vert A^{\rm{abs}}\Vert_\infty .
\end{equation}
For other norms, then  the norm of a matrix may be  different form the norm of its absolute value.
Consider for example
$
A=\left(
\begin{array}{cc}
1 & 1 \\
-1 & 1
\end{array}
\right)
$
together with the euclidian norm
$
\|A\|_2$.  Then it is easy to check that 
$\|A\|_2=\sqrt 2< \| A^{\rm{abs}} \|_2=2$.
 We  introduce two evident lemmas used later in the rest of the Section.

\begin{lemma}\label{lem:prod_of_abs}
Let $A_1,A_2,\dots,A_N$ be $N$ matrices  which are compatible in the sense  that the product $A_N A_{N-1}\dots A_2 A_1\in{\cal M}_{m,n}(\R)$ is well defined. Then
one has 
\[
\left|
\left(A_N A_{N-1}\dots A_2 A_1\right)_{i,j} \right|\le \left(A_N^{\rm{abs}} A_{N-1}^{\rm{abs}}\dots A_2^{\rm{abs}} A_1^{\rm{abs}}\right)_{i,j}
\qquad \mbox{for all } 1\leq i \leq m \mbox{ and } 1\leq j \leq n
\]
and
$
\left\| A_N A_{N-1}\dots A_2 A_1\right\|\leq
\left\|
A_N^{\rm{abs}} A_{N-1}^{\rm{abs}}\dots A_2^{\rm{abs}} A_1^{\rm{abs}}
\right\|$. 
\end{lemma}
\begin{proof}
The lemma is proved by expressing the entries of the product of matrices in terms of the entries $a_{i,j}^r$ of the initial matrices $A_r$, $r=1,\dots,N$.
One checks that 
\begin{equation*}
\begin{aligned}
\left(A_N A_{N-1}\dots A_2 A_1\right)_{i,j}^{\rm{abs}}  &= \left| \sum_{k_1,k_2,\dots,k_{N-1}}a_{i,k_{N-1}}^N a_{k_{N-1},k_{N-2}}^{N-1}\dots a_{k_2,k_1}^2 a_{k_1,j}^1\right| 
\\
&\le  \sum_{k_1,k_2,\dots,k_{N-1}}\left|a_{i,k_{N-1}}^N a_{k_{N-1},k_{N-2}}^{N-1}\dots a_{k_2,k_1}^2 a_{k_1,j}^1\right| \\
&=  \sum_{k_1,k_2,\dots,k_{N-1}}\left|a_{i,k_{N-1}}^N \right| \left|a_{k_{N-1},k_{N-2}}^{N-1}\right| \dots \left|a_{k_2,k_1}^2 \right| \left|a_{k_1,j}^1\right| \\
&= \left(A_N^{\rm{abs}} A_{N-1}^{\rm{abs}}\dots A_2^{\rm{abs}} A_1^{\rm{abs}}\right)_{i,j}. 
\end{aligned}
\end{equation*}
The definitions of the $l^1$ norm and of the $l^\infty$ norm yield the second inequality.
\end{proof}

\begin{lemma}\label{lem:prod_of_pos}
Take matrices $A,C\in{\cal M}_{m,n}(\R)$ and $B,D\in{\cal M}_{n,p}(\R)$ be 
such that
$
0\leq A_{i,j} \le C_{i,j}$ and $ 0\leq B_{i,j} \le D_{i,j} $
for all  values of $i$ and $j$. Then one has 
$
\left(AB\right)_{i,j} \le \left(CD\right)_{i,j}$, $\forall\, i=1,\dots,m,\,\, j=1,\dots,p,
$
and
$\left\| AB \right\| \leq \left\| CD \right\| $.
\end{lemma}

\begin{proof}
Indeed  one has 
$
\left(AB\right)_{i,j} = \sum_{k=1}^n A_{i,k}B_{k,j}\le \sum_{k=1}^n C_{i,k}D_{k,j}= \left(CD\right)_{i,j}$ for all $ i=1,\dots,m$ and $j=1,\dots,p$. 
The definitions of the $l^1$ and $l^\infty$ norms yield the second inequality.
\end{proof}


\begin{theorem}\label{theo:abs_bound}
One has the bound $
K\le K_3$ where 
$
K_3= \Vert W_\ell^{\rm{abs}}  W_{\ell-1}^{\rm{abs}} \dots W_0^{\rm{abs}} \Vert.
$
Moreover,  this new bound  is sharper than the worst bound, that is $K_3\leq K_\star$.
\end{theorem}

\begin{proof}
The proof is made for  the $l^\infty$ norm. One has 
\begin{equation}\label{eq:proof_kabs_infty}
\begin{aligned}
K &= \max_{D\in{\cal{D}}} \Vert W_\ell D_\ell W_{\ell-1}D_{\ell-1}\dots D_1 W_0 \Vert_\infty\\
&= \max_{D\in{\cal{D}}} \max_i \sum_j \left\vert\sum_{k_\ell,\dots,k_1} w_{i,k_\ell}^\ell d_{k_\ell}^\ell w_{k_\ell,k_{\ell-1}}^{\ell-1} \dots
d_{k_1}^1 w_{k_1,j}^0 \right\vert \\
&\le\max_{D\in{\cal{D}}} \max_i \sum_j \sum_{k_\ell,\dots,k_1}\left\vert w_{i,k_\ell}^\ell d_{k_\ell}^\ell  w_{k_\ell,k_{\ell-1}}^{\ell-1}\dots
d_{k_1}^1 w_{k_1,j}^0 \right\vert \\
&\le \max_i \sum_j \sum_{k_\ell,\dots,k_1}\left\vert w_{i,k_\ell}^\ell\right\vert\left\vert w_{k_\ell,k_{\ell-1}}^{\ell-1}\right\vert\dots
\left \vert w_{k_1,j}^0 \right\vert 
= \Vert W_\ell^{\rm{abs}}  W_{\ell-1}^{\rm{abs}} \dots W_0^{\rm{abs}} \Vert_\infty.
\end{aligned}
\end{equation}
 The proof  for  the $l^1$ norm   is derived by exchanging  $i$ and $j$ in \eqref{eq:proof_kabs_infty}. It yields the first part of the claim.
 
The second part is obtained as follows. One has $
K_3 = \Vert W_\ell^{\rm{abs}}  W_{\ell-1}^{\rm{abs}} \dots W_0^{\rm{abs}} \Vert \le \Vert W_\ell^{\rm{abs}}\Vert \Vert W_{\ell-1}^{\rm{abs}}\Vert\dots \Vert W_0^{\rm{abs}}\Vert$
by  sub-multiplicativity of the norm.
Using \eqref{eq:norm_equiv_per_abs}  one has  
$  \Vert W_\ell^{\rm{abs}}\Vert \Vert W_{\ell-1}^{\rm{abs}}\Vert\dots \Vert W_0^{\rm{abs}}\Vert 
 =  \Vert W_\ell\Vert \Vert W_{\ell-1}\Vert\dots \Vert W_0\Vert$,
 so $K_3 \leq  K_*$.
\end{proof}

Next, we combine the Combettes-Pesquet technique with the bound $K_3$ and our choice of norm (either $\Vert\cdot\Vert_1$ or $\Vert\cdot\Vert_\infty$). 
It 
defines  a new bound 
\begin{equation}\label{eq:kabscp_def}
K_{4} = \dfrac{1}{2^\ell} \sum_{1\le r_1< r_2<\dots< r_n\le\ell}  \Vert W_{(\ell+1,r_{n})}^{\rm{abs}}\dots W_{(r_2,r_1)}^{\rm{abs}} W_{(r_1,0)}^{\rm{abs}}\Vert.
\end{equation}

\begin{theorem} \label{theo:fcnn_kabscp}
One has the bounds
\begin{equation}\label{eq:abscp_cp_ineq}
K\le K_4\le K_1.
\end{equation}
\end{theorem}

\begin{proof}
We start by proving $K\le K_4$. 
The third line in  (\ref{eq:k3}) rewrites as  
\begin{equation} \label{eq:kb1}
 K \leq \dfrac{1}{2^\ell} \max_{Z\in{\cal{Z}}}\sum_{1\le r_1< r_2<\dots< r_n\le\ell} \Vert W_{(\ell+1,r_{n})}Z_{r_n} W_{(r_{n},r_{n-1})} \dots
 W_{(r_2,r_1)}Z_{r_1} W_{(r_1,0)}\Vert.
\end{equation}
Lemma \ref{lem:prod_of_abs} shows that
\begin{equation} \label{eq:kb2}
\Vert W_{(\ell+1,r_{n})}Z_{r_n} W_{(r_{n},r_{n-1})} \dots
 W_{(r_2,r_1)}Z_{r_1} W_{(r_1,0)}\Vert
 \leq
 \Vert W_{(\ell+1,r_{n})}^{\rm abs}Z_{r_n} ^{\rm abs}W_{(r_{n},r_{n-1})} ^{\rm abs}\dots
 W_{(r_2,r_1)}^{\rm abs}Z_{r_1}^{\rm abs} W_{(r_1,0)}^{\rm abs}\Vert.
\end{equation}
 Generically, one has $\left( Z_{r} ^{\rm abs}\right)_{ii}\leq 1$ for all $i$ and $\left( Z_{r} ^{\rm abs}\right)_{ij}=0$ for all  $i\neq j$. 
 That is  the  matrix $Z_r$ is element-wise dominated  by the identity matrix of the same size.
 Lemma \ref{lem:prod_of_pos} used recursively yields that
 \begin{equation} \label{eq:kb3}
 \Vert W_{(\ell+1,r_{n})}^{\rm abs}Z_{r_n} ^{\rm abs}W_{(r_{n},r_{n-1})} ^{\rm abs}\dots
 W_{(r_2,r_1)}^{\rm abs}Z_{r_1}^{\rm abs} W_{(r_1,0)}^{\rm abs}\Vert
 \leq
 \Vert W_{(\ell+1,r_{n})}^{\rm abs}W_{(r_{n},r_{n-1})} ^{\rm abs}\dots
 W_{(r_2,r_1)}^{\rm abs}  W_{(r_1,0)}^{\rm abs}\Vert.
 \end{equation}
The combination of (\ref{eq:kb1}--\ref{eq:kb3}) yields the claim.


{The inequality $K_4\le K_1$ comes from
\begin{equation*}
\begin{aligned}
K_{4} & \le\dfrac{1}{2^\ell} \sum_{1\le r_1< r_2<\dots< r_n\le\ell}  \Vert W_{(\ell+1,r_{n})}^{\rm{abs}}\Vert\dots\Vert W_{(r_2,r_1)}^{\rm{abs}}\Vert\Vert W_{(r_1,0)}^{\rm{abs}}\Vert \\
&\le\dfrac{1}{2^\ell} \sum_{1\le r_1< r_2<\dots< r_n\le\ell}  \Vert W_{(\ell+1,r_{n})}\Vert\dots\Vert W_{(r_2,r_1)}\Vert\Vert W_{(r_1,0)}\Vert = K_1
\end{aligned}
\end{equation*}
thanks to the norm submultiplicativity and \eqref{eq:norm_equiv_per_abs}.
}
\end{proof}

\begin{proposition}\label{prop:k4_le_k3}
One has the bound
$
K_4 \le K_3$. 
\end{proposition}

\begin{proof}
Consider the generic term in the sum of $K_4$ 
\begin{equation}\label{eq:abs_cp_gen_term}
\Vert W_{(\ell+1,r_{n})}^{\rm{abs}}\dots W_{(r_2,r_1)}^{\rm{abs}} W_{(r_1,0)}^{\rm{abs}}\Vert
\end{equation}
for suitable indices $1\le r_1<\dots< r_n\le\ell$. Then consider the generic matrix $W_{r_{i+1},r_i}^{\rm{abs}}$ in \eqref{eq:abs_cp_gen_term}, $i=0,\dots,n$, where $r_0=0$ and $r_{n+1}=\ell+1$. For all admissible indices $i,j$, the following inequality holds as a consequence of Lemma \ref{lem:prod_of_abs}
\begin{equation}\label{eq:abs_le_prod_of_abs}
\left(W_{r_{k+1},r_k}^{\rm{abs}}\right)_{i,j} 
=
\left|
\left(W_{r_{k+1},r_k}\right)_{i,j} 
\right|
\le \left( W_{r_{k+1}-1}^{\rm{abs}} W_{r_{k+1}-2}^{\rm{abs}}\dots W_{r_{k}+1}^{\rm{abs}}W_{r_k}^{\rm{abs}}\right)_{i,j}.
\end{equation}
 We can recursively use Lemma \ref{lem:prod_of_pos} to claim that 
\begin{equation}\label{eq:single_norm_ineq_abscp_le_abs}
\Vert W_{(\ell+1,r_{n})}^{\rm{abs}}\dots W_{(r_2,r_1)}^{\rm{abs}} W_{(r_1,0)}^{\rm{abs}}\Vert \le \Vert W_\ell^{\rm{abs}}\dots W_1^{\rm{abs}}W_0^{\rm{abs}}\Vert .
\end{equation}
Therefore
\begin{equation*}
\begin{aligned}
K_4 &= \dfrac{1}{2^\ell} \sum_{1\le r_1< r_2<\dots< r_n\le\ell}  \Vert W_{(\ell+1,r_{n})}^{\rm{abs}}\dots W_{(r_2,r_1)}^{\rm{abs}} W_{(r_1,0)}^{\rm{abs}}\Vert \\
&\le \dfrac{1}{2^\ell} \sum_{1\le r_1< r_2<\dots< r_n\le\ell}  \Vert W_\ell^{\rm{abs}}\dots W_1^{\rm{abs}}W_0^{\rm{abs}}\Vert
= \dfrac{1}{2^\ell} 2^\ell \Vert W_\ell^{\rm{abs}}\dots W_1^{\rm{abs}}W_0^{\rm{abs}}\Vert
= K_3,
\end{aligned}
\end{equation*}
where we use the definition of $K_3$ and the fact that the sum comprises $2^\ell$ terms that can be bounded by the same quantity independent from the subscripts $r_1,\dots, r_n$.
\end{proof}

In summary, we derived the series of inequalities
$$
L\leq K\leq K_4\leq \min \left(K_1, K_3 \right) \le  \max \left(K_1, K_3 \right)\leq K_\star.
$$
With this respect $K_4$ is a better upper bound than the others. How much better? It will be evaluated in the numerical Section with basic numerical experiments.

\section{Convolutional neural networks}\label{sec:conv}

  A CNN is  a specialized kind of neural network \cite{goodi,desp:book} for processing data that have a known, grid-like topology.
  CNNs    are tremendously successful in practical applications. 
It is therefore appealing to examine
 under what conditions the previous material, developed for fully-connected feed-forward neural networks,  can be generalized to CNNs. 
A generic   CNN function \cite{li2021survey}  is written under the form 
\begin{equation}\label{eq:conv_def}
g = 
 g_\ell \circ T_\ell \circ g_{\ell-1} \circ T_{\ell-1} \circ\dots \circ T_1 \circ g_0.
\end{equation}
This structure formally seems very similar to the one of feed-forward neural networks (\ref{eq:fcnn_def}) but there are important differences which are summarized in the next Remarks.

\begin{remark}\label{rem:remi1}
In  (\ref{eq:fcnn_def})  the $f_r$ are linear and have ability to change the dimension (of the underlying space),
 while the $S_r$ are non linear activation  functions and  do not change the dimension. 
In (\ref{eq:conv_def}) the notation cannot be as simple.
We choose arbitrarily that the $g_r$  have ability to change the dimension and that the $T_r$ do not change the dimension.
It means that some $g_r$ can be non linear  while some $T_r$ can be linear. 
\end{remark}

\begin{remark} \label{rem:remi2}
The main example that motivates the classification made in above Remark  \ref{rem:remi1} is the max-pooling function detailed below. Indeed the max-pooling function   is non linear and changes the dimension
so it cannot fit
in the simple structure (\ref{eq:fcnn_def}). Since a max-pooling function $g_r$ is, to the best of our knowledge, rarely followed by a non linear function, it is necessary
to allow the functions $T_{r+1}$ to be equal to the identity if it is needed, that is $T_{r+1}=I$. Then it makes  $T_{r+1}$ a linear function.
In summary all $g_r$ and all $T_r$  can be linear or non linear in  our modeling of  convolutional neural networks.
\end{remark}

\begin{remark}\label{rem:remi2_v2}
As max-pooling layers, also average-pooling layers are rarely followed by non linear activation functions. In this case, the operator $g_{r+1}\circ T_r\circ g_r$, composed by an average-pooling layer $g_r$, the identity operator $T_r$ and another linear operator $g_{r+1}$, is linear and can be exactly represented by a single linear layer which matrix is the product between the matrices associated with $g_{r+1}$ and $g_r$. Since this merging operation is always possible and preserve the structure \eqref{eq:conv_def}, in the following we always assume that all triplets of consecutive linear operators are already merged, i.e. if $g_{r+1}$ and $g_r$ are linear, then $T_r$ is non linear.
\end{remark}

\begin{remark} \label{rem:remi3}
It is common for CNN  \cite{goodi} to finish with another non linear
 function  $T_{\ell+1}$ called softmax function \cite{desp:book}, because it is very effective  for classification tasks.
We just disregard  $T_{\ell+1}$  for the simplicity of notation.
It  brings no additional difficulty to keep it because its Lipschitz constant is naturally equal to  1. 
\end{remark}

\begin{remark} \label{rem:remi4}
CNN are mainly used for classification purposes, contrary to feed-forward  neural networks which are mainly used for regression purposes.
For classification, it is highly possible that one tries to approximate (by tuning the coefficient of the networks) a  function with low regularity.
In this case, it is not clear what the Lipschitz constant of the function modeled by (\ref{eq:conv_def}) should be.
Nevertheless it is easy to imagine simple examples which are relevant for classification where the objective function is smooth (in particular if a thresholding post-processing
mechanism
is added). In this case smoothness could as well bring some stability in the approximation/training process.
That is why we think that it is a relevant mathematical question to inquire about the Lipschitz constant of convolutional neural networks.
\end{remark}

%
%
%

 Let us now describe more precisely the different functions we have in mind.


The functions  $g_i$, $i=0,\dots,\ell$ can be either linear functions denoted   as $g_i^{\rm{lin}}(x)=W_i x + b_i$ (as for a feed-forward neural networks), or
convolutional functions denoted by $g_i^{\rm{conv}}$, or average-pooling functions denoted by $g_i^{\rm{avg\_pool}}$ (a pooling function is a particular 
averaging function). These functions being all linear, they do not bring conceptually new material with respect to feed-forward neural networks.
Another type of non linear  function $g_i^{\rm{max\_pool}}$, named max-pooling function, will be introduced below.
All these different functions $g_i$ have the ability to change the dimension of the data: that is $g_i: \mathbb R^{a_i}\to \mathbb R^{a_{i+1}}$ where
$a_{i+1}\neq a_i$ is a possibility.
The  functions
  $T_i$, $i=1,\dots,\ell$ can be any  activation function or the identity function. 
  All these different functions $T_i$ do not change  the dimension of the data: that is $T_i: \mathbb R^{a_{i+1}}\to \mathbb R^{a_{i+1}}$.
  As in the previous Section,  the non linear activation functions $T_i$ are 
  such that  $0\le T_i'(x)\le 1$ for almost all  $ x\in\R$ and are applied component wise.
  \\
A linear function $g_i^{\rm{lin}}: \mathbb R^{a_i}\to \mathbb R^{a_{i+1}}$ is written as before:
$$
g_i^{\rm{lin}}(x) = W_ix+b_i.
$$
 A convolutional  function  $g_i^{\rm{conv}}: \mathbb R^{a_i}\to \mathbb R^{a_{i+1}}$ operates  the convolution between the input and a convolution  matrix $K_i^{\rm{conv}}$, with the addition of  a  bias
\[
g_i^{\rm{conv}}(x) = K_i^{\rm{conv}} * \overline x +\overline b_i . 
\]
Since convolution operators are linear operators, there exists a matrix $W_i=W_i^{\rm{conv}}$ such that
\[
g_i^{\rm{conv}}(x) = W_i  x +  b_i,  
\]
where $ x$ and $ b_i$ are vectors obtained by serializing/reindexing the vectors $\overline x$ and $\overline b_i$. 
With this notation,  $W_i^{\rm{conv}}$ is a double circulant matrix (this can be considered as a practical definition of a convolution operator in our context).
 The correspondance between $x$ and $\overline x$ is explained in more details  in the following example.
Let us take a vertical vector
\begin{equation} \label{eq:x_s1}
x=
\begin{bmatrix}
    x_{11} &  x_{12} & x_{13} &
    x_{21} & x_{22} & x_{23} &
    x_{31} & x_{32} & x_{33} 
\end{bmatrix}^T.
\end{equation}
A re-indexation  allows to write 
\begin{equation} \label{eq:x_s2}
 \overline x = 
\begin{bmatrix}
    x_{11} & x_{12} & x_{13} \\
    x_{21} & x_{22} & x_{23} \\
    x_{31} & x_{32} & x_{33} 
\end{bmatrix}.
\end{equation}

An average-pooling function $g_i^{\rm{avg\_pool}}: \mathbb R^{a_i}\to \mathbb R^{a_{i+1}}$ calculates the average of $x$ on patches which correspond to a certain multidimensional structure
of the data stored in $x$. Since it is a linear operator, there exist a matrix $W_i^{\rm{avg\_pool}}$ such that $g_i^{\rm{avg\_pool}}(x) = W_i^{\rm{avg\_pool}} x$ as for convolutional layers. An average-pooling function  with filter size (2,2) and stride 1 applied to the matrix
correspond to the averaging of all $2\times 2$ sub-matrix possible in (\ref{eq:x_s2}).
It yields 
\begin{equation} \label{eq:ggg}
g_i^{\rm{avg\_pool}}( \overline x) =
\begin{bmatrix}
    \avg(x_{11}, x_{12}, x_{21}, x_{22}) &  \avg(x_{12}, x_{13}, x_{22}, x_{23})  \\
    \avg(x_{21}, x_{22}, x_{31}, x_{32}) &  \avg(x_{22}, x_{23}, x_{32}, x_{33}) 
\end{bmatrix}
\end{equation}
where $  \avg(a,b,c,d)=\frac14(a+b+c+d) $.
The very same function can of course be 
defined for the vector $x$ in (\ref{eq:x_s1}).
One obtains
\begin{equation}\label{eq:avg_pool_mat}
g_i^{\rm{avg\_pool}}(x) = 
\begin{bmatrix}
    \frac14 &\frac14 &0 &\frac14 &\frac14 &0 &0 &0 &0 \\[0.2cm]
    0 &\frac14 &\frac14 &0 &\frac14 &\frac14 &0 &0 &0  \\[0.2cm]
    0& 0& 0& \frac14 &\frac14 &0 &\frac14 &\frac14 &0  \\[0.2cm]
    0& 0& 0& 0& \frac14 &\frac14 &0 &\frac14 &\frac14 
\end{bmatrix}
\begin{bmatrix}
    x_{11} \\  x_{12} \\ x_{13} \\
    x_{21} \\ x_{22} \\ x_{23} \\
    x_{31} \\ x_{32} \\ x_{33} 
\end{bmatrix}.
\end{equation}
Since the difference between (\ref{eq:ggg}) and (\ref{eq:avg_pool_mat}) is just a matter of re-indexation, these two functions are  the same
and they are identified. The same re-indexation and identification process is used for convolution functions \cite{goodi,desp:book}. \\
The  max-pooling operation of function is the most original one since it introduces a non linearity.
Using the notation (\ref{eq:x_s2}), it is written as 
\begin{equation} \label{eq:x_s5}
g_i^{\rm{max\_pool}}(\overline x) = 
\begin{bmatrix}
    \max(x_{11}, x_{12}, x_{21}, x_{22}) &  \max(x_{12}, x_{13}, x_{22}, x_{23})  \\
    \max(x_{21}, x_{22}, x_{31}, x_{32}) &  \max(x_{22}, x_{23}, x_{32}, x_{33}) 
\end{bmatrix}.
\end{equation}
 By construction,   a CNN (\ref{eq:conv_def}) is a continuous and globally Lipschitz  
 $$
 g\in C^0(\mathbb R^{a_0})^{a_{\ell+1}}\bigcap \mbox{Lip}(\mathbb R^{a_0})^{a_{\ell+1}}
 $$ 
and the Rademacher Theorem  is still valid.
 The rest of the Section is devoted to generalize  to CNNs the bounds $K$, $K_\star$, $K_1$, $K_3$ and $K_4$ already  developed for  
 fully-connected feed-forward neural networks.
 There is no additional difficulty for linear functions and for classical  activation functions, that is for $g_i^{\rm{lin}}$,
 $g_i^{\rm{conv}}$, $g_i^{\rm{avg\_pool}}$ and $T_i$. The technical difference is for the max-pooling functions $g_i^{\rm{max\_pool}}$.
 We distinguish two approaches called the explicit approach and the implicit approach.
 
To be consistent with Section \ref{sec:new_bounds}, we restrict our analysis to the norms $\Vert\cdot\Vert_1$ and $\Vert\cdot\Vert_\infty$ and we denote by $\Vert\cdot\Vert$ any of these two norms. We also highlight that we do not consider the bound $K_2$ since the matrices associated with convolutional neural networks are too large and it would be impossible to compute such bound as explained in Remark \ref{rem:scb}.

\subsection{Explicit approach}\label{sec:expl_app}

The first method that we consider for the decomposition of a max-pooling function $g_i^{\rm{max\_pool}}$ is based on the functional identity
\begin{equation}\label{eq:max_abs_rel}
\begin{aligned}
\max(x_1, x_2) &= \frac12(x_1+x_2)+\frac12\left|x_1-x_2\right|\\
&=\frac12
\begin{bmatrix}
    1&1 
\end{bmatrix}
\begin{bmatrix}
    x_1\\x_2
\end{bmatrix} + \frac12 \abs\left(
\begin{bmatrix}
    1&-1 
\end{bmatrix}
\begin{bmatrix}
    x_1\\x_2
\end{bmatrix}\right)\\
&=\frac12
\begin{bmatrix}
    1&1 
\end{bmatrix}
\begin{bmatrix}
    x_1\\x_2
\end{bmatrix} + \frac12 z
\begin{bmatrix}
    1&-1 
\end{bmatrix}
\begin{bmatrix}
    x_1\\x_2
\end{bmatrix},
\end{aligned}
\end{equation}
where $z=\pm 1$.
This decomposition is the generalization to max-pooling layers of the decomposition (\ref{eq:dtoz}) which was considered for an activation function in the Combettes-Pesquet approach \cite{combettes2020lipschitz}.
 We call it \textit{explicit} because the linear 
 matrices $\begin{bmatrix}
    1&1 
\end{bmatrix}$  and $\begin{bmatrix}
    1&-1 
\end{bmatrix}$ are explicit. The non linear part depends on just one coefficient $z$.

An interesting property is linked to this decomposition. Let us consider $ u(x_1,x_2):=\max(x_1,x_2)$.
We also consider $v(x_1,x_2)=x_1+x_2$ which corresponds to the first part of the decomposition (\ref{eq:max_abs_rel})
and $w(x_1,x_2)=x_1-x_2$ which corresponds to the second part of the decomposition.
We write
$u=\frac12 v +\frac12 zw $ with $ z=\pm 1$.
The gradient of $u$ is (almost everywhere)
 $\nabla u=\begin{bmatrix}
    1&0
\end{bmatrix}$ or $\nabla u=\begin{bmatrix}
    0&1
\end{bmatrix}$.
Also $\nabla v=\begin{bmatrix}
    1&1
\end{bmatrix}$ and $\nabla w=\begin{bmatrix}
    1&-1
\end{bmatrix}$.
So one has the  relation  for the gradients (almost everywhere with respect to $x$)
\begin{equation} \label{eq:uvw}
\nabla u=\frac12 \nabla v +\frac12 z \nabla w 
\end{equation}
with $ z=\pm 1$.
Comparing with (\ref{eq:max_abs_rel}), one obtains the identity
$u(x)=  \nabla u (x)\cdot x $ which is 
 the Euler identity for homogeneous function of degree one (the function $ u(x_1,x_2):=\max(x_1,x_2)$ is indeed  homogeneous of degree one).
Then the triangular inequality yields (once again almost everywhere with respect to $x$)
\begin{equation} \label{eq:tuvw}
\| \nabla u \|_{l^p} \leq \frac 12 \| \nabla v \|_{l^p} +\frac12 \| \nabla w \|_{l^p}
\end{equation}
where the $l^p$ norm is the induced norm for operators, $1\leq p \leq \infty$.

\begin{lemma} \label{lem:3.1}
The $l^1$ norm is optimal for the triangular inequality (\ref{eq:tuvw}), in the sense that it is an equality for $p=1$ and a strict inequality for $1<p\leq \infty$.
\end{lemma}

\begin{proof}
 Clearly $L(u)=\left( \| \nabla u \|_{l^p}\right)_{L^\infty(\mathbb R^2)}=1$.
The $l^p$ norm of the operator $\nabla v=\begin{bmatrix}
    1&1
\end{bmatrix}$ is evaluated with (\ref{eq:valval}). It yields 
$$
L(v)=
 \| \nabla v \|_{l^p}= \max_{(a,b)\neq 0} \frac{|a+b|}{\left(|a|^p +|b|^p \right)^\frac1p}  
 \leq 
  \max_{(a,b)\neq 0} \frac{\left(|a|^p +|b|^p \right)^\frac1p \left(2 \right)^\frac1q}{\left(|a|^p +|b|^p \right)^\frac1p}=2^\frac1q
$$
where we used a H\"older inequality and $q$ is  conjugate to $p$, that is $\frac 1p+\frac1q=1$. This is actually  an equality since the H\"older inequality
is optimal, so $L(v)=2^\frac1q$. For the same reason $L(w)=2^\frac1q$.
Then the triangular inequality $L(u)\leq \frac12 L(v)+\frac12 L(w)$ reduces to
$
1\leq 2^\frac1q\Longleftrightarrow 1\leq 2^{\frac{p-1}p}$. This inequality is an equality only for $p=1$ which yields the claim.
\end{proof}

 This method can be applied recursively to compute a max-pooling function of any structure.
 A first application is as follows. 
 
  \begin{lemma} \label{lem:mn}
 The representation (\ref{eq:max_abs_rel}) composed $n$ times yields the function
$
\max(x_1, x_2,\dots, x_n)
$.
 A bound on the $l^p$ norm  of the gradient  is $2^{\frac{(p-1)(n-1)}p}$.
 \end{lemma}
 
 \begin{proof}
 Write  $\max(x_1, x_2,\dots, x_n)= \max(x_1, \max( x_2,, \dots ))$ and iterate Lemma \ref{lem:3.1}. 
 \end{proof}
 
   A second  application is max-pooling combined with convolution. 
  Instead of presenting a complicate theory which will bring very little in terms of ideas, let us consider just one example  which operates on either (\ref{eq:x_s1})
  or (\ref{eq:x_s2}) with kernel size $(2,1)$ and stride 1 (row-wise max-pooling). So the pooling is made with every  pairs in (\ref{eq:x_s2}) on the same {\bf row}. One writes 
\begin{equation}\label{eq:horiz_conv}
g_{\rm{row}}^{\rm{max\_pool}}(x) =\frac12 M_{\rm{row}}^{+} x + \frac12 Z_{\rm{row}}^{\rm{max\_pool}}(x)M_{\rm{row}}^{-}  x
\end{equation}
where $ x\in \mathbb R^9$ is given in (\ref{eq:x_s1}) and 
\begin{equation*}
\begin{aligned}
 M_{\rm{row}}^{+}=
\begin{bmatrix}
    1 &1 &0 &0 &0 &0 &0 &0 &0 \\[0.2cm]
    0 &1 &1 &0 &0 &0 &0 &0 &0  \\[0.2cm]
    0& 0& 0& 1 &1 &0 &0 &0 &0  \\[0.2cm]
    0& 0& 0& 0& 1 &1 &0 &0 &0 \\[0.2cm]
    0& 0& 0& 0 &0 &0 &1 &1 &0  \\[0.2cm]
    0& 0& 0& 0& 0 &0 &0 &1 &1 
\end{bmatrix}, \quad 
M_{\rm{row}}^{-} =
\begin{bmatrix}
    1 &-1 &0 &0 &0 &0 &0 &0 &0 \\[0.2cm]
    0 &1 &-1 &0 &0 &0 &0 &0 &0  \\[0.2cm]
    0& 0& 0& 1 &-1 &0 &0 &0 &0  \\[0.2cm]
    0& 0& 0& 0& 1 &-1 &0 &0 &0 \\[0.2cm]
    0& 0& 0& 0 &0 &0 &1 &-1 &0  \\[0.2cm]
    0& 0& 0& 0& 0 &0 &0 &1 &-1 
\end{bmatrix}
\end{aligned}
\end{equation*}
and  $Z_{\rm{row}}^{\rm{max\_pool}}(x)\in{\cal M}_{6,6}(\R)$ is a diagonal matrix with diagonal elements $1$ or $-1$. 
The result is a vector of size $6$.

Analogously, a max-pooling layer  is  performed on {\bf columns}
with kernel size $(1,2)$ and stride 1 (column-wise max-pooling) applied on the resulting vector  $x\in \mathbb R^6$. One writes 
\begin{equation}\label{eq:vert_conv}
g_{\rm{col}}^{\rm{max\_pool}}(x) = 
\frac12 M_{\rm{col}}^{+} x + \frac12 Z_{\rm{col}}^{\rm{max\_pool}}(x)M_{\rm{col}}^{-}  x
\end{equation}
where $x=
\begin{bmatrix}
    x_{11} &  x_{12}  &
    x_{21} & x_{22}  &
    x_{31} & x_{32}  
\end{bmatrix}^T\in \mathbb R^6$ and 
\begin{equation*}
\begin{aligned}
M_{\rm{col}}^{+}=
\begin{bmatrix}
    1 &0  &1 &0 &0 &0 \\[0.2cm]
    0 &1 &0 & 1 &0 &0   \\[0.2cm]
    0& 0& 1 &0 &1 &0  \\[0.2cm]
    0& 0& 0& 1& 0 &1 
\end{bmatrix},
\quad M_{\rm{col}}^{-} =
\begin{bmatrix}
    1 &0  &-1 &0 &0 &0 \\[0.2cm]
    0 &1 &0 & -1 &0 &0   \\[0.2cm]
    0& 0& 1 &0 &-1 &0  \\[0.2cm]
    0& 0& 0& 1& 0 &-1 

\end{bmatrix}
\end{aligned}
\end{equation*}
where $Z_{\rm{col}}^{\rm{max\_pool}}(x)\in{\cal M}_{4,4}(\R)$ is a diagonal matrix with elements $1$ and $-1$.
The composition yields a max-pooling function  $g^{\rm{max\_pool}}:\mathbb R^{9}\to \mathbb R^4$ (with kernel size $(2,2)$) 
\begin{equation}\label{eq:pooling_split}
g^{\rm{max\_pool}}(x) = g_{\rm{col}}^{\rm{max\_pool}} \circ g_{\rm{row}}^{\rm{max\_pool}}(x) 
= \left(\frac12 M_{\rm{col}}^{+} + \frac12 Z_{\rm{col}}^{\rm{max\_pool}}(x)M_{\rm{col}}^{-} \right)\left(\frac12 M_{\rm{row}}^{+} + \frac12 Z_{\rm{row}}^{\rm{max\_pool}}(x)M_{\rm{row}}^{-} \right)
  x
\end{equation}
The gradient of the function is (almost everywhere in $x$)
\begin{equation}\label{eq:pooling_split:g}
\nabla g^{\rm{max\_pool}}(x) 
= \left(\frac12 M_{\rm{col}}^{+} + \frac12 Z_{\rm{col}}^{\rm{max\_pool}}(x)M_{\rm{col}}^{-} \right)\left(\frac12 M_{\rm{row}}^{+} + \frac12 Z_{\rm{row}}^{\rm{max\_pool}}(x)M_{\rm{row}}^{-} \right).
\end{equation}
As for the example (\ref{eq:uvw}), we observe that the identity $ g^{\rm{max\_pool}}(x)= \nabla g^{\rm{max\_pool}}(x) \cdot x$ holds, 
which comes from the fact that $ g^{\rm{max\_pool}}(x)$ is homogenous of degree one with respect to $x$.

To simplify the notations, we will develop the theory for $(2,2)$ max-pooling kernels described by (\ref{eq:pooling_split:g}).
The general case will be treated in Remarks.


\begin{remark} \label{rem:mn}
Max-pooling functions with arbitrary kernel size $(m,n)$ can be expressed as the compositions of $m-1$ row-wise max-pooling layers and $n-1$ column-wise max-pooling layers.
A justification is  Lemma \ref{lem:mn}.
It can be extended to any number of dimensions.
It constructs max-pooling functions with a  control of the Lipschitz constant with respect to the $l^p$ norm.
\end{remark}

\begin{remark}
In view of Lemma \ref{lem:3.1}, it is appealing to use systematically the $l^1$ norm to get free of extra multiplicative constants for $l^p$ norms with $p>1$.
This may  probably become a sensitive issue if the  recursive structure is used for  kernel sizes $(m,n)$ with either $m\gg1$ or $n\gg 1$.
\end{remark}

Let us consider 
$D_r(x) = T_r'\circ g_{r-1}\circ T_{r-1}\circ\dots\circ T_1\circ g_0 (x)\in \mathcal D_r$, which is a square diagonal matrix.
In the case the function $g_r$ is linear we note
$W_r=\nabla g_r$. If  $g_r$ is a  max-pooling function with kernel $(2,2)$ under the form 
(\ref{eq:pooling_split}) we note
$Z_{{\rm row},r}^{\rm{max\_pool}}(x)$ and $Z_{{\rm col},r}^{\rm{max\_pool}}(x)$ the two diagonal matrices that show up.
We do note consider kernels larger than $(2,2)$ for the simplicity of notation but they are immediate to treat using Remark \ref{rem:mn}.
The main point is that these matrices $Z_{{\rm row},r}^{\rm{max\_pool}}(x)$ and $Z_{{\rm col},r}^{\rm{max\_pool}}(x)$ play a role similar to the matrices $D_r(x)$, so they must be treated as well
with the Combettes-Pesquet trick.
We note
$$
\left\{Z_{{\rm col},r}^{\rm{max\_pool}}(x), Z_{{\rm row},r}^{\rm{max\_pool}}(x)\right\}
\in 
\mathcal Z_r ^{\rm{max\_pool}}
$$ 
where $\mathcal Z_r^{\rm{max\_pool}}$ is the ensemble
of all pairs of diagonal matrices of convenient size  used for max-pooling layers with diagonal coefficients equal to $\pm 1$.
Finally, we note 
$$
Z_{{\rm col}}(x)=
\left( \left\{Z_{{\rm col},r}^{\rm{max\_pool}}(x), Z_{{\rm row},r}^{\rm{max\_pool}}(x)\right\}
 \right)_r\in 
\mathcal Z^{\rm{max\_pool}}=\Pi_{r} \mathcal  Z_r^{\rm{max\_pool}}
$$
where the indices are restricted of course  to max-pooling layers. 
We also denote by $W$ the collection of all remaining weight matrices for layers representing linear operators, that is 
$
W=(\dots, W_r, \dots) $ where $\ell\geq r \geq 0$ and $g_r$  is not a max-pooling layer. The ensemble of all these weights is the vectorial spaces
$\mathcal W$ such that $W\in \mathcal W$.

We now have the set of notations needed to study the Lipschitz constant of the CNN function  (\ref{eq:conv_def}).
The chain rule yields
\begin{equation} \label{eq:ng}
\nabla g= Y_\ell(x) D_\ell (x)Y_{\ell-1}(x)D_{\ell-1}(x)\dots D_1 (x)Y_0(x)
\end{equation}
where $D(x)=(D_\ell (x), \dots, D_0 (x)) \in \mathcal D$ and $Y_r(x)$ is a matrix which is as follows.
Either $Y_r(x)=W_r$, which is the layer's matrix of weights and it is constant with respect to $x$, or 
\begin{equation} \label{eq:yr}
Y_r(x)=\left(\frac12 M_{\rm{col},r}^{+} + \frac12 Z_{\rm{col},r}^{\rm{max\_pool}}(x)M_{\rm{col},r}^{-} \right)\left(\frac12 M_{\rm{row},r}^{+} + \frac12 Z_{\rm{row},r}^{\rm{max\_pool}}(x)M_{\rm{row},r}^{-} \right)
\end{equation}
 which depends on $x$ through the matrices $\left( Z_{\rm{col},r}^{\rm{max\_pool}}(x), Z_{\rm{row},r}^{\rm{max\_pool}}(x)\right)\in \mathcal Z_r ^{\rm{max\_pool}}
$.
From the chain rule identity (\ref{eq:ng}) one defines the Lipschitz constant
$$
L=
\left( \| \nabla g(x)\|_{l^p \left(\mathcal M_{a_{\ell+1},a_0}(\mathbb R) \right)} \right) _{L^\infty(\mathbb R^{a_0})}.
$$
For given weight matrices $W\in \mathcal W$, one has by definition  
$
L\leq K
$ where
\begin{equation} \label{eq:wouf}
K^{\rm expl}=K^{\rm expl}(W)=\max_{{(D,Z)\in{\mathcal{D}}\times {\mathcal{Z}}^{\rm{max\_pool}}}} \left\|
Y_\ell D_\ell Y_{\ell-1}D_{\ell-1}\dots D_1 Y_0
\right\|
\end{equation}
where either $Y_r=W_r$ or
$
Y_r=\left(\frac12 M_{\rm{col},r}^{+} + \frac12 Z_{\rm{col},r}^{\rm{max\_pool}}M_{\rm{col},r}^{-} \right)\left(\frac12 M_{\rm{row},r}^{+} + \frac12 Z_{\rm{row},r}^{\rm{max\_pool}}M_{\rm{row},r}^{-} \right)$, where the degrees of freedom are $(Z_{\rm{col},r}^{\rm{max\_pool}}, Z_{\rm{row},r}^{\rm{max\_pool}})
\in  \mathcal Z^{\rm{max\_pool}} _r $ because the other matrices are given.
This first bound is referred to as the explicit one since it is based on the explicit representation (\ref{eq:max_abs_rel}). 

\subsection{Explicit CNN bounds}\label{sec:expl_cnn_bounds}

The comparison between the expression (\ref{eq:wouf}) for CNNs and the expres\-sion (\ref{eq:fcnn_starting_bound}) for feed-forward dense networks 
shows that the only  difference is in the     additional set of matrices $ \mathcal Z^{\rm{max\_pool}}$. 
 Apart from the inflation of  notational issues due to the explicit matrices $M_{\rm{col},r}^{+}$, $M_{\rm{col},r}^{-}$, $ M_{\rm{row},r}^{+}$ and $ M_{\rm{row},r}^{-}$
(for relevant max-pooling layers $r$), the results are fundamentally similar to the ones in Sections \ref{sec:prev_works_fcnn} and 
\ref{sec:new_bounds}.
 The main idea is explained in   the following example 
 \begin{equation} \label{eq:grad11}
 g=g_4\circ T_4 \circ g_3 \circ T_3 \circ g_2 \circ T_2 \circ g_1 \circ T_1 \circ g_0.
 \end{equation}
 We make the same assumptions as for the example (\ref{eq:gradf1})  except that $g_2$ is now a max-pooling layer with kernel size $(2,2)$.
 As explained in Remark \ref{rem:remi2}, the operator $T_3$ is now the identity.
 With the notation (\ref{eq:yr}), the gradient formula (\ref{eq:gradf}) becomes
 \begin{equation} \label{eq:grad12}
 \nabla g=\frac1{2^5}\   W_4 (I+Z_4)  W_3 \left((M^+_{\rm{col},2} +Z^{\rm{max\_pool}} _{\rm{col}, 2} M^-_{\rm{col},2})
 (M^+_{\rm{rwo},2} +Z^{\rm{max\_pool}} _{\rm{rwo}, 2} M^-_{\rm{row},2})
  \right) (I+Z_2) W_1 (I+Z_1) W_0
 \end{equation}
 where the identity matrices may  have different sizes as  in the general case.
 All matrices of type  $Z$ are square matrices with $\pm 1$ on the diagonal, so their norm is equal to 1.
 Also one sees that a matrix $Z^{\rm{max\_pool}} _{\rm{col}, 2} $ or  $Z^{\rm{max\_pool}} _{\rm{rwo}, 2} $ is necessarily on the left of a matrix
 $M^-_{\rm{col},2}$ or $M^-_{\rm{row},2}$. 
%
%

In order to generalize this example, ,  we define two matrices $W^+$ and $W^-$
for every   linear, convolution, average-pooling or  max-pooling layer.
For linear layers we set  $W^+=W^-=W^{\rm{lin}}$.
For convolutional layers we set $W^+=W^-=W^{\rm{conv}}$.
For average-pooling layers we set $W^+=W^-=W^{\rm{avg\_pool}}$.
For row-wise max-pooling layers (see eq. (\ref{eq:yr})) we set  $W^+=M_{\rm row}^+$, $W^-=M_{\rm row}^-$ and for column-wise max-pooling layers we set $W^+=M_{\rm col}^+$, $W^-=M_{\rm col}^-$. 
 Note  that max-pooling layers are fundamentally
considered as two different non-linear operations, which explains why we make a distinction between {\bf row}-like and {\bf column}-like matrices.
We set
$$
\ell'=\ell+\mbox{number of max-pooling layers}.
$$
and we express each two-dimensional max-pooling layer as a combination of one-dimensional max-pooling layers,
 we obtain a neural network with $\ell'\ge \ell$ layers.\\
We define 
\begin{equation} \label{eq:w-}
W_{[t,s]} = W_{t-1}^- W_{t-2}^+W_{t-3}^+ \dots W_{s+2}^+W_{s+1}^+ W_{s}^+
\end{equation}
where the ordering is the natural one for a neural network with $\ell'$ layers.

Note that $W_{[t,s]}$ and $W_{(t,s)}$ coincide for fully connected neural networks since $W^+=W^-$ for linear layers.
We consider 
\[
 K_*^{\rm{expl}} := \prod_{i=0}^{\ell'} \Vert W_i^+\Vert.
\]

\begin{lemma}
One has $K^{\rm{expl}} \le  K_*^{\rm{expl}} $ (restricted to $l^1$ and  $l^\infty$ norms).
\end{lemma}

\begin{proof}
Consider the typical matrices of a max-pooling layer (\ref{eq:yr}).
The key property is $\| M_{\rm{col},r}^{+}\|= \| M_{\rm{col},r}^{-}\|$ for the $l^1$ and the $l^\infty$ norms.
Since $\| Z_{\rm{col},r}^{\rm{max\_pool}}\|=1$ by construction, then
$
\left\| \frac12 M_{\rm{col},r}^{+} + \frac12 Z_{\rm{col},r}^{\rm{max\_pool}}(x)M_{\rm{col},r}^{-} \right\| \leq \| M_{\rm{col},r}^{+}\|
= \| W_r^+\|.
$
The rest of the proof is based on the sub-multiplicativity of  norms.
\end{proof}

Now we define 
\begin{equation}\label{eq:def_cp_conv}
 K_1^{\rm{expl}} = \dfrac{1}{2^{\ell'}} \sum_{
 1\le r_1< r_2<\dots< r_n\le\ell' 
} \Vert W_{[\ell'+1,r_{n}]}\Vert\Vert W_{[r_{n},r_{n-1}]}\Vert\dots
\Vert W_{[r_2,r_1]}\Vert  \Vert W_{[r_1,0]}\Vert,
\end{equation}

\begin{theorem} 
\label{theo:expl_cb_bound}
One has $K^{\rm{expl}}\le K_1^{\rm{expl}}{\le K_*^{\rm{expl}}}$.
\end{theorem}

\begin{proof}
The proof {of the first inequality} is a direct generalization of the one of Theorem \ref{def:kcp}.
It is based on a direct  expansion of (\ref{eq:ng}--\ref{eq:yr}). The matrices $Z_r(x)$ (coming from $D_r(x)$), the matrices
$Z^{\rm max\_pool}_{{\rm col},r}$ and the matrices $Z^{\rm max\_pool}_{{\rm row},r}$ divide all terms in the expansion between chunks 
of matrices. 
One can check that our notations are such that a matrix $W^-$ is always just before a matrix $Z$ (of any of the previous types).
The rest is a matter of direct calculus. 

{The second inequality is proved as in Proposition \ref{prop:k1_sharper_kstar}, exploiting the fact that $\| W_r^{+}\|= \| W_r^{-}\|$ for one-dimensional max-pooling layers.}
\end{proof}

\begin{remark}[Extension to  general max-pooling layers of size $(m,n)$]
Such results hold for the convolutive structure (\ref{eq:yr}) which corresponds to the $(2,2)$ convolutive kernel
for which we gave the details of the construction in (\ref{eq:horiz_conv}--\ref{eq:pooling_split}). 
As explained in Remark \ref{rem:mn}, it is sufficient to compose with additional {\bf row}-wise and {\bf column}-wise matrices
to model two-dimensional max-pooling layers of any size. It is easy to generalize to max-pooling kernels in higher dimensions
$(m,n,o,p\dots)$.
The number $\ell'$ of terms in the formulas must be changed of course. The general formula 
is the sum of the numbers of terms for a feed-forward fully-connected neural networks, plus all the additional layers needed to represent
the CNN.
\end{remark}

Next we define
$$
K_3^{\rm{expl}} := \left\Vert \prod_{i=0}^{\ell'} \left(W_i^+\right)^{\rm{abs}}\right\Vert \le K_*^{\rm{expl}}.
$$

\begin{lemma}
One has 
\begin{equation}\label{eq:expl_abs}
K^{\rm{expl}} \le K_3^{\rm{expl}} {\le K_*^{\rm{expl}}. 
}
\end{equation}
\end{lemma}

\begin{proof}
For $*\in\{{\rm{row}},{\rm{col}}\}$ one has 
\begin{equation*}
\begin{aligned}
\left(\frac12 M_{\rm{*}}^{+} + \frac12 Z_{\rm{*}}M_{\rm{*}}^{-} \right)_{ij}^{\rm{abs}} &\le 
\frac12 \left(M_{\rm{*}}^{+}\right)_{ij}^{\rm{abs}} + \frac12 \left[Z_{\rm{*}}^{\rm{abs}}\left(M_{\rm{*}}^{-} \right)^{\rm{abs}} \right]_{ij}\\
&= \frac12 \left(M_{\rm{*}}^{+}\right)_{ij}^{\rm{abs}} + \frac12 \left[I\left(M_{\rm{*}}^{+} \right)^{\rm{abs}} \right]_{ij}\\
&= \left(M_*^+\right)_{ij}^{\rm{abs}}.
\end{aligned}
\end{equation*}
Then, the inequalities in \eqref{eq:expl_abs} follow from theorem \ref{theo:abs_bound}.
\end{proof}

 We define the following quantity
\[
K_4^{\rm{expl}}:=\dfrac{1}{2^{\ell'}} \sum_{1\le r_1< r_2<\dots< r_n\le\ell'} \left\Vert W_{[\ell'+1,r_n]}^{\rm{abs}} W_{[r_n,r_{n-1}]}^{\rm{abs}} \dots  W_{[r_2,r_1]}^{\rm{abs}} W_{[r_1,0]}^{\rm{abs}}\right\Vert.
\]

\begin{theorem} 
One has 
$
K^{\rm{expl}} \le K_4^{\rm{expl}} \le K_1^{\rm{expl}}$ {and $K_4^{\rm{expl}} \le K_3^{\rm{expl}}$ 
}. 
\end{theorem}

\begin{proof}
Same proofs as in Theorem \ref{theo:fcnn_kabscp} {and in Proposition \ref{prop:k4_le_k3}}.
\end{proof}

\subsection{Implicit approach}\label{sec:impl_app}

The interest of what we call the implicit approach is that it simplifies the implementation for the max-pooling functions, because
it is a one step technique for the modeling of $\max(x_1,\dots,x_n)$.
It avoids the  recursive technique used 
in Lemma \ref{lem:mn}. Moreover, this leads to bounds that are computationally more efficient because the number of terms in the sum in $K_1$ and $K_4$ grows exponentially with the number of layers.

Consider the set ${\cal{X}} = \{x_1,x_2,\dots,x_n\}$, $n\ge2$. Let $x_M$ be the maximum element in $\cal{X}$ and let ${\cal{X}}^{-M} =  {\cal{X}} \backslash\{x_M\}$ be the set of the remaining elements. Then, $\max(x_1,\dots,x_n)$ can be computed as
\begin{equation}\label{eq:impl_start_eq}
\max(x_1,\dots,x_n) = \frac12 (x_1+x_2+\dots+x_n) + \frac12\left(x_M - \sum_{x_i\in{\cal{X}}^{-M}}x_i\right).
\end{equation}
We consider two  matrices ${\mathds{1}}=(1, \dots,1)\in \mathbb R^{1\times n}$ and 
${{\mathds{Z}}}_M=(-1, \dots, -1, +1, -1, \dots, -1 )\in \mathbb R^{1\times n}$.
The matrix ${{\mathds{Z}}}_M$ is made only of coefficients -1 except one (at position $M$) which is equal to +1.
Next, we modify the notations at the beginning of Section \ref{sec:expl_app}.
 Let us denote $ u(x_1,\dots, x_n):=\max(x_1,\dots, x_n)$,
 $v(x_1,\dots, x_n)={\mathds{1}} x$ 
 and $w(x_1,\dots, x_n)={\mathds{Z}}_Mx$. 
One has 
$u=\frac12 v +\frac12 w $.  .
Then, almost everywhere,   $\nabla u=\frac12 {\mathds{1}}+\frac12 {\mathds{Z}} _M $,  $\nabla v=  {\mathds{1}}
$ and $\nabla w= {\mathds{Z}}_M$.
So one has the  relation  for the gradients (almost everywhere with respect to $x$)
$\nabla u=\frac12 \nabla v +\frac12  \nabla w $.
We call this approach the \textit{implicit one} because $ {\mathds{Z}}_M$ depends on $M$.

The generalization is as follows, but 
once again, instead of a complicate theory, we explain how to use this decomposition for the example  (\ref{eq:x_s1})--(\ref{eq:x_s2})
for the max-polling function (\ref{eq:x_s5}). One simply has to modify the matrices 
${\mathds{1}}$ and ${\mathds{Z}}_M$.
One notices that
 a max-pooling function like  (\ref{eq:x_s5}) satisfies two identities (almost everywhere), which are
 \begin{equation} \label{eq:g(x)}
 g( x)=\frac12 {\mathds{1}} x +\frac12 {\mathds{Z}}_M x 
 \end{equation}
 and
 $$
 \nabla g( x)=\frac12  {\mathds{1}}  +\frac12  {\mathds{Z}}_M. 
 $$
 Here one has
 $$
   {\mathds{1}}=\begin{bmatrix}
    1 &1 &0 &1 &1 &0 &0 &0 &0 \\[0.2cm]
    0 &1 &1 &0 &1 &1 &0 &0 &0  \\[0.2cm]
    0& 0& 0& 1 &1 &0 &1 &1 &0  \\[0.2cm]
    0& 0& 0& 0& 1 &1 &0 &1 &1
\end{bmatrix}.
$$
Concerning the matrix $ {\mathds{Z}}_M$ there is at most  16 different possibilities (so $ M< 16$) 
since all four lines of $M$ display coefficients equal to 0 or -1, except that one -1 per line is changed into +1.
The number of different possibilities is not equal to 16 because  there are some redundancies between lines. 
An example is
$$
 {\mathds{Z}}_M=\begin{bmatrix}
    1 &-1 &0 &-1 &-1 &0 &0 &0 &0 \\[0.2cm]
    0 &-1 &-1 &0 &1 &-1 &0 &0 &0  \\[0.2cm]
    0& 0& 0& -1 &-1 &0 &-1 &1 &0  \\[0.2cm]
    0& 0& 0& 0& -1 &-1 &0 &1 &-1
\end{bmatrix}.
$$

A trivial property, important  for further developments, holds for such matrices.

\begin{lemma} 
\label{lem:impi}
One has  ${\mathds{Z}}_M^{\rm abs}= {\mathds{1}} $ which does not depend on $M$. Similarly, one has 
$\| {\mathds{Z}}_M \|=\| {\mathds{1}} \|$. 
\end{lemma}
\begin{proof}
Consider the matrices $ {\mathds{1}}$ and ${\mathds{Z}}_M$ of the above example. In this case, the lemma is trivial. It is the same in the general case.
\end{proof}

It is now simple  to implement this representation of the max-pooling functions and their gradients in (\ref{eq:ng}).
Instead of (\ref{eq:yr}) one takes for max-pooling functions only 
$
Y_r(x)=\frac12  {\mathds{1}}_r  +\frac12  {\mathds{Z}}_{M,r}$.
Then one sets
\begin{equation} \label{eq:wouf-2}
K^{\rm impl}=K^{\rm impl}(W)=\max 
 \left\|
Y_\ell D_\ell Y_{\ell-1}D_{\ell-1}\dots D_1 Y_0
\right\|
\end{equation}
where either $Y_r=W_r$ or
$
Y_r=\frac12  {\mathds{1}}_r  +\frac12  {\mathds{Z}}_{M,r}$. 
The maximum is taken over all possible matrices $D_r$ and all possible matrices ${\mathds{Z}}_{M,r}$.

The theoretical simplification  offered by  the implicit approach is clear since
only one matrix ${\mathds{Z}}_{M,r}$ (with nevertheless possibly different $M$) is enough to represent the gradient of max-pooling functions
for any kernels, instead of the recursive representations explained in Remark \ref{rem:mn}.

\begin{remark}
In the implicit approach, a max-pooling layer $g_r$ is expressed as the sum of a term $\frac12\mathds{1}_r$ independent from the input and a term $\frac12{\mathds{Z}}_{M,r}$ which depends on the input. A similar expansion, as the sum of a term $\frac12 M_{*,r}^{+}$ independent from the input and a term $\frac12 Z_{\rm{*},r}^{\rm{max\_pool}}M_{\rm{*},r}^{-}$ depending on the input, is introduced in Section \ref{sec:expl_app} for the one-dimensional max-pooling layer $g_{*}^{\rm{max\_pool}} $, $*\in\{{\rm{row}},{\rm{col}}\}$. The key difference between these two expansions is that the dependency on the input in the explicit approach is represented by a diagonal matrix with diagonal entries equal to 1 or -1, which is multiplied by a fixed known matrix. In the implicit approach, instead, the term depending on the input is a single rectangular matrix with norm greater than 1. This difference leads to the following alternative upper bounds.
\end{remark}



\subsection{Implicit  CNN bounds}

Since the implicit approach has a structure which is similar to the explicit one, we just write the different bounds and terms without further explanations.

Consider once again the example (\ref{eq:grad11}) where the  max-pooling  function $g_2$ is modeled with the implicit approach (\ref{eq:g(x)}).
The gradient is represented as 
$$
 \nabla g=\frac1{2^4}\   W_4 (I+Z_4)  W_3 \left( {\mathds{1}}_2+ {\mathds{Z}}_{M,2}
  \right) (I+Z_2) W_1 (I+Z_1) W_0.
$$
The norm of the gradient is obtain after full expansion and use of Lemma \ref{lem:impi}.

To model a more general  two-dimensional max-pooling layer $g_r$ we denote by $W_r^+={\mathds{1}}_r$ and $W_r^-= {\mathds{Z}}_{M,r}$.
We  introduce the matrices 
\begin{equation}\label{eq:def_w_graffe}
W_{\{t,s\}}=\left\{
\begin{aligned}
&\hphantom{W_{t-1}^-}W_{t-2}^+W_{t-3}^+ \dots W_{s+1}^+ W_{s}^+ &&\hspace{0.3cm} {\text{\rm{if }}} g_{t-1}{\text{\rm{ is a max-pooling layer }}}, \\
&W_{t-1}^- W_{t-2}^+W_{t-3}^+ \dots W_{s+1}^+ W_{s}^+ &&\hspace{0.3cm} {\text{\rm{otherwise}}} ,
\end{aligned}\right.
\end{equation}
To have a compact notation, we set 
\begin{equation}\label{eq:def_R}
R_{t}=\left\{
\begin{aligned}
&W_{t-1}^- && \hspace{0.3cm}{\text{\rm{if }}} g_{t-1}{\text{\rm{ is a max-pooling layer }}}, \\
&I_{t-1} &&\hspace{0.3cm} {\text{\rm{otherwise}}}, 
\end{aligned}\right.
\end{equation}
where $I_{t-1}$ is an identity matrix with as many rows as $W_{t-1}^-$. 
To prove the following bounds, it is sufficient to extend the bounds in Section \ref{sec:expl_cnn_bounds} by using, for each max-pooling layer $g_r$, the quantities $\left(W^-_r\right)^{\rm{abs}}$ and $\Vert W^-_r\Vert$ instead of $W_r^-$.

The following bound holds: 
\[
K^{\rm{impl}} \le K_*^{\rm{impl}} := \prod_{i=0}^{\ell} \Vert W_i^+\Vert.
\]
One also has
\begin{equation}\label{eq:def_cp_conv_impl}
\begin{aligned}
K^{\rm{impl}}\le K_1^{\rm{impl}} = \dfrac{1}{2^{\ell}} \sum_{1\le r_1< r_2<\dots< r_n\le\ell} \Vert W_{\{\ell+1,r_{n}\}}&\Vert\Vert R_{r_n}\Vert\Vert W_{\{r_{n},r_{n-1}\}}\Vert\Vert R_{r_{n-1}}\Vert\dots\\
&\dots\Vert R_{r_2}\Vert
\Vert W_{\{r_2,r_1\}}\Vert  \Vert R_{r_1}\Vert\Vert W_{\{r_1,0\}}\Vert.
\end{aligned}
\end{equation}
One has
\begin{equation}\label{eq:impl_abs}
K^{\rm{impl}} \le K_3^{\rm{impl}} := \left\Vert \prod_{i=0}^{\ell} \left(W_i^+\right)^{\rm{abs}}\right\Vert \le K_*^{\rm{impl}}.
\end{equation}
In view of our general results, the most interesting quantity is 
\[
K_4^{\rm{impl}}:=\dfrac{1}{2^{\ell'}} \sum_{1\le r_1< r_2<\dots< r_n\le\ell} \left\Vert W_{\{\ell+1,r_n\}}^{\rm{abs}} R_{r_n}^{\rm{abs}} W_{\{r_n,r_{n-1}\}}^{\rm{abs}} R_{r_{n-1}}^{\rm{abs}}\dots R_{r_2}^{\rm{abs}} W_{\{r_2,r_1\}}^{\rm{abs}} R_{r_1}^{\rm{abs}}W_{\{r_1,0\}}^{\rm{abs}}\right\Vert.
\]
The following inequalities hold
\[
K^{\rm{impl}} \le K_4^{\rm{impl}} \le K_1^{\rm{impl}}\,\, \mbox{ and }\,\,{K_4^{\rm{impl}} \le K_3^{\rm{impl}}}.
\]

\section{Numerical results}\label{sec:num_res}

We provide and discuss some numerical experiments to illustrate  the theoretical results presented in the previous sections. The numerical results correspond to
 a variety of  very different test problems that come from different origins.
The first test problem is the evaluation of the Lipschitz constant for a fully-connected neural network with three layers with random matrices proposed in  \cite{combettes2020lipschitz}.
The second test problem is particular in the sense that one has a reference solution \cite{yarotsky2017error,daubi,desp:book} for the neural network approximation of
$x\mapsto x^2$. It  allows for sound and simple numerical comparisons, see also some developments in \cite{despi1,despi2}.
The third test problem is a new reference solution  that we propose  for the  function $(x,y)\mapsto xy$.
The fourth problem is the now standard CNN for the  MNIST problem \cite{lecu}. All the numerical tests have been implemented in Tensorflow \cite{tensorflow2015-whitepaper}.

\subsection{Neural networks with random weights}
In this section, we repeat the numerical test with random matrices and weights
proposed in Combettes-Pesquet  \cite[Example 2.1]{combettes2020lipschitz}.
 One  considers a neural network with the following architecture
\begin{equation}\label{eq:test_net}
f = f_3 \circ R \circ f_2 \circ R \circ f_1.
\end{equation}
Here  $R:\R\rightarrow\R$ denotes the ReLU activation function, that is  $R(x)=\max(0,x)$.
The matrices associated with the linear layers are $W_1 \in {\cal M}_{10,8}(\R)$, $W_2 \in {\cal M}_{6,10}(\R)$ and $W_3 \in {\cal M}_{3,6}(\R)$, their entries are i.i.d. realizations of the normal distribution ${\cal{N}}(0,1)$. As in \cite{combettes2020lipschitz}, we initialize 1000 neural networks and analyse the behaviour of the bounds. If the $l^2$ norm is used, the results presented in \cite{combettes2020lipschitz} are obtained (not reproduced here).  
To be coherent with the theory developed in Section \ref{sec:new_bounds} with a more general approach, we prefer to present
the results  computed  with respect to the  $l^\infty$ norm. 

Note that the smaller is a bound, the closer it is to the true Lipschitz constant. The same behaviour holds for the normalized bounds shown in Table \ref{tab:avg_max_stat}, where each bound is divided by $K_*$. In particular, when such ratio is close to 1, it means that the considered bound is not significantly better than the naive bound $K_*$. On the other hand, when it is close to $K/K_*$, it means that the bound is almost as sharp as the ideal (but usually incomputable) bound $K$.

The obtained statistics are reported in Table \ref{tab:avg_max_stat}. Due to the small dimensionality  of the considered network, it is possible to explicitly compute the ideal bound $K$ in \eqref{eq:fcnn_starting_bound}. Coherently with the theory, we observe
 the  inequalities $K\le K_4\le K_3\le K_*$ and $K_4\le K_1\le K_*$. Moreover $K_1$ and $K_3$ are the bounds with the lowest and highest variances (respectively). The smaller value of the ratio $K_\dag/K_*$, $\dag\in\{1,2,3,4\}$, is obtained using the bound $K_2$. 

\begin{table}[htbp]
\footnotesize
\caption{Statistics over 1000 realizations of the network in \eqref{eq:test_net}.}\label{tab:avg_max_stat}
\begin{center}
  \begin{tabular}{|c|c|c|c|c|c|} \hline
   Statistic & $K/K_*$ & $K_1/K_*$ & $K_2/K_*$ & $K_3/K_*$ & $K_4/K_*$ \\ \hline
    Maximum & 0.2772 & 0.5786 & 0.6789 &0.8023  & 0.4608 \\
    Average & 0.1422  & 0.4539 & 0.3256 & 0.5461 & 0.2875  \\
    Minimum & 0.0595 & 0.3703 & 0.1597 & 0.2897 & 0.1604 \\
    Standard deviation & 0.0343 & 0.0350 & 0.0685 & 0.0813 & 0.0483 \\ \hline
  \end{tabular}
\end{center}
\end{table}

\subsection{Approximation of $x^2$}\label{sec:x2}

We present a neural network with an arbitrary number of layers which output is known in a simple closed form over their entire domain
 \cite{yarotsky2017error,daubi,desp:book}. Understanding these types of models, where multiple layers and nonlinearities are present but the typical uncontrolled oscillations of deep neural networks are absent, is crucial in order to improve the deep learning mathematical theory.
Let us consider the one-dimensional hat function $g:[0,1]\rightarrow[0,1]$
\begin{equation*}
g(x) = 
\left\{
\begin{aligned}
&2x, &&\,\,x\in[0,0.5),\\
&2(1-x), &&\,\, x\in[0.5, 1]. 
\end{aligned}
\right.
\end{equation*}
The function $g_r:[0,1]\rightarrow[0,1]$ obtained by composing $r$ times  the function $g$ with itself
\[
g_r(x) = \underbrace{g\circ g\circ \dots \circ g\circ g}_{r{\rm{\text{ times}}}}.
\]
The functions $g_1$, $g_2$ and $g_3$ are represented in Figure \ref{fig:g1_g2_g3}.
\begin{figure}[t!]
\centering 
  \includegraphics[width=0.5\columnwidth,keepaspectratio,clip]{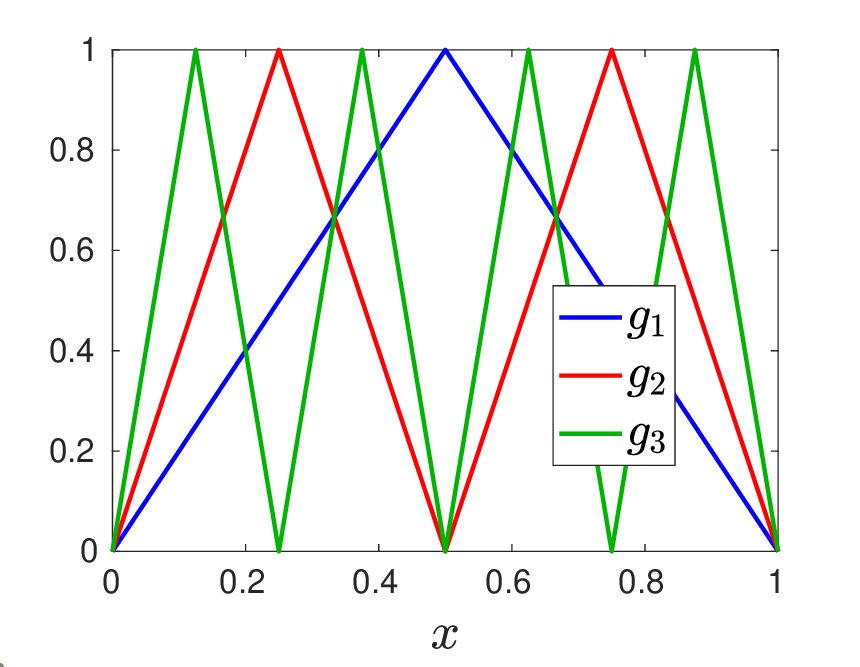} 
  \caption{Graphical representation of $g_1$, $g_2$ and $g_3$.}
  \label{fig:g1_g2_g3}
\end{figure}
The series $\sum_{r=0}^\infty \frac{g_r(x)}{4^r}$ converges to the function $x-x^2$ (two different proves are proposed in \cite{desp:book}). It is then possible to express $x^2$ as
\begin{equation}\label{eq:takagi_series}
x^2 = x - \sum_{r=0}^\infty \frac{g_r(x)}{4^r}.
\end{equation}
Since the function $g$ can be exactly expressed as a small ReLU network and the series on the right hand side of \eqref{eq:takagi_series} converges exponentially fast, it is possible to efficiently approximate the squaring function with a deep ReLU network in the following way.
\begin{itemize}
\item Construct a ReLU network which output coincides with the function $g$. Different choices are available. For example, if we use ReLU networks with two layers and two neurons in the internal one, $g$ can be expressed as
\begin{equation}\label{eq:g_vers1}
g(x) = 1 - R(2x-1) - R(1-2x), \hspace{1cm}  x\in[0,1]
\end{equation}
or
\begin{equation}\label{eq:g_vers2}
g(x) = R(2x) - R(4x-2), \hspace{1cm}  x\in[0,1].
\end{equation}
The function in \eqref{eq:g_vers1} can be represented as a two-layers ReLU neural network $g(x)=f_1\circ R\circ f_0(x)$ with weight matrices and vectors:
\begin{equation}\label{eq:mat_vec_vers1}
W_0 = 
\begin{bmatrix}
    2 \\[0.2cm]
    -2 
\end{bmatrix},\hspace{0.5cm}
b_0 = 
\begin{bmatrix}
    -1 \\[0.2cm]
    1 
\end{bmatrix},\hspace{0.5cm}
W_1 = 
\begin{bmatrix}
    -1, -1
\end{bmatrix},\hspace{0.5cm}
b_1 = 
\begin{bmatrix}
    1
\end{bmatrix}.
\end{equation}
Analogously, the function in \eqref{eq:g_vers2} can be represented by a neural network with the same architecture but with the following weights:
\begin{equation}\label{eq:mat_vec_vers2}
W_0 = 
\begin{bmatrix}
    2 \\[0.2cm]
    4 
\end{bmatrix},\hspace{0.5cm}
b_0 = 
\begin{bmatrix}
    0 \\[0.2cm]
    -2 
\end{bmatrix},\hspace{0.5cm}
W_1 = 
\begin{bmatrix}
    1, -1
\end{bmatrix},\hspace{0.5cm}
b_1 = 
\begin{bmatrix}
    0
\end{bmatrix}.
\end{equation}
\item Concatenate the networks representing the function $g$ and merge each pair of consecutive linear layers into a single new linear layer. For the sake of clarity, we assume that the same representation of the function $g$ (e.g. the one in \eqref{eq:g_vers1} or \eqref{eq:g_vers2}) is adopted on every layer. However such a constraint is not necessary. 
Then the function $g_3$ can be computed as
\begin{equation*}
\begin{aligned}
g_3 &= g\circ g \circ g \\&= (f_1\circ R\circ f_0)\circ(f_1\circ R\circ f_0)\circ(f_1\circ R\circ f_0) \\
&= f_1\circ R\circ f_{0,1}\circ R\circ f_{0,1}\circ R\circ f_0,
\end{aligned}
\end{equation*}
where the new layer $f_{0,1}$ is naturally defined as:
\begin{align*}
f_{0,1}(x) &= f_0\circ f_1(x) \\
&= W_0 (W_1 x+b_1) + b_0\\
&= W_0W_1 x + W_0b_1+b_0\\
&= W_{0,1} x + b_{0,1},
\end{align*}
with $W_{0,1} = W_0W_1$ and $b_{0,1}=W_0b_1+b_0$.

\item Add one neuron to each layer to evaluate and add all the terms in the series in equation \eqref{eq:takagi_series}.
It yields a collation channel \cite{daubi}.
 The specific rows and columns that have to be added to the weight matrices depend on the chosen representation of the function $g$. The resulting neural network associated with the function $x - \sum_{r=0}^3 \frac{g_r(x)}{4^r}$ is represented in Figure \ref{fig:x2_network}.
\end{itemize}

\begin{figure}[t!]
\centering 
  \includegraphics[width=0.5\columnwidth,keepaspectratio,clip]{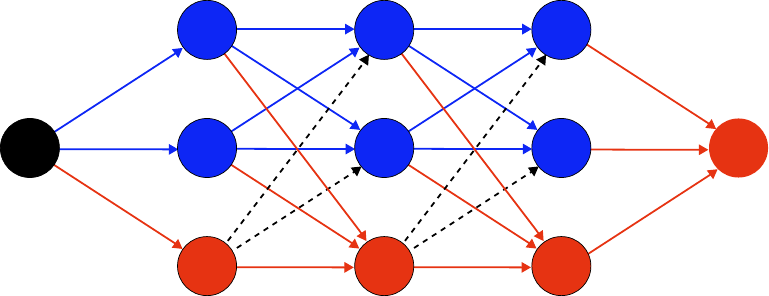} 
  \caption{Graphical representation of the network representing the function $x - \sum_{r=0}^3 \frac{g_r(x)}{4^r}$. The blue dots and edges are associated with the construction of the function $g$. The red ones are used to store and sum $x$ and $-g_i/4^i$, $i=1,2,3$. The weight associated with the dashed lines is 0.}
  \label{fig:x2_network}
\end{figure}

We now evaluate the bounds presented in Section \ref{sec:fcnn} on this specific networks. The obtained bounds are shown in Figure \ref{fig:bounds_per_layers}.  In the left subplot, the function $g$ is represented as in \eqref{eq:g_vers1}. It is possible to observe that $K_*$, $K_1$ and $K_2$ grow exponentially with the number of layers, whereas $K_3$ grows only linearly and $K_4$ is exact, as highlighted in Table \ref{tab:lips_for_x2} for the first networks. In such table, the exact Lipschitz constant $L$ is computable due to the specific network structure. Instead, if the representation in \eqref{eq:g_vers2} is chosen, all bounds grow exponentially, apart from $K_4$ which grows linearly. In both cases, coherently with the theory, the following inequalities hold: $K_4 \le K_1 \le K_*$, $K_4 \le K_3 \le K_*$.

\begin{figure}[t!]
\centering 
\begin{subfigure}[t]{0.45\linewidth}
  \includegraphics[width=\columnwidth,keepaspectratio,clip]{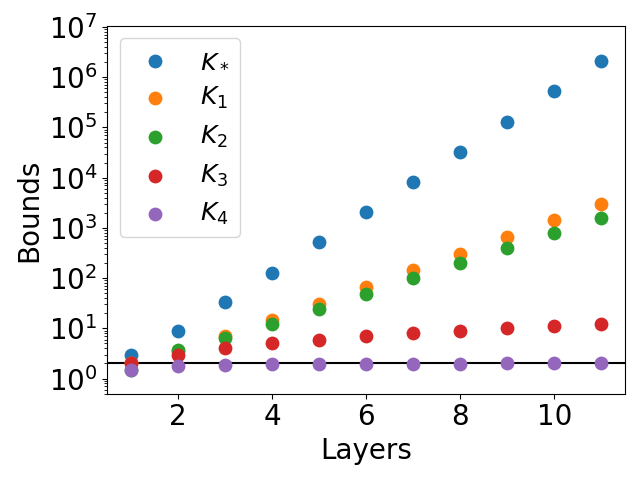} 
  \end{subfigure}
\begin{subfigure}[t]{0.45\linewidth}
  \includegraphics[width=\columnwidth,keepaspectratio,clip]{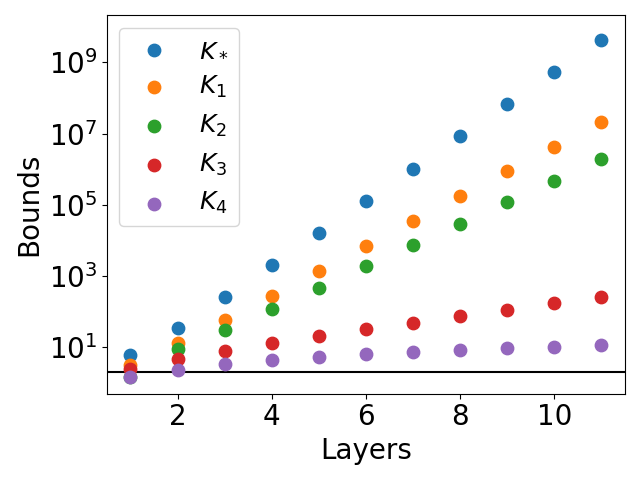} 
  \end{subfigure}
  \caption{Lipschitz bounds for networks approximating the function $x^2$. The function $g$ is represented as in \eqref{eq:g_vers1} (left) or as in \eqref{eq:g_vers2} (right).}
  \label{fig:bounds_per_layers}
\end{figure}

\begin{table}[htbp]
 \caption{Exact Lipshitz constants and upper bounds for networks approximating the function $x^2$. The function $g$ is represented as in \eqref{eq:g_vers1}.}
\begin{center}
\begin{tabular}{|c||c|c|c|c|c|c|}
\hline
$\ell$ & $L$ & $K_*$ & $K_1$ & $K_2$ & $K_3$ & $K_4$ \\ \hline\hline
1 &  1.5 & 3.0 & 2.0 & 1.5 & 2.0 & 1.5 \\ \hline
2 & 1.75 & 9.0 & 3.53125 & 3.64531 & 3.0  & 1.75 \\ \hline
3 &  1.875 & 33.0 & 6.92187 & 6.45925 & 4.0 & 1.875 \\ \hline
4 &  1.93875 & 129.0 & 14.42877 & 12.45882 & 5.0  & 1.93875 \\ \hline 
5 & 1.96875 & 513.0 & 30.75637 & 24.68184 & 6.0 & 1.96875\\ \hline
6 &1.984375  & 2049.0 & 66.00227 & 49.24501 & 7.0 & 1.984375 \\ \hline
\end{tabular}
\end{center}
 \label{tab:lips_for_x2}
\end{table}

Denoting by $K_\dag^\ell$ (for $\dag\in\{*,1,2,3,4\}$) the evaluation of one of the proposed bounds for a network with $\ell$ layers, we define the rate of growth of $K_\dag$ as
$
G_\dag^\ell = \frac{K_\dag^{\ell+2}-K_\dag^{\ell+1}}{K_\dag^{\ell+1}-K_\dag^{\ell}}.
$
This quantity is  useful  to understand the behavior of the growths shown in Fig \ref{fig:bounds_per_layers}. In fact, $K_\dag$ grows exponentially when $G_\dag^\ell$ is constant and larger than 1, $K_\dag$ grows linearly when $G_\dag^\ell=1$, and $K_\dag$ converges exponentially to a constant value when $G_\dag^\ell$ is constant and smaller than 1. The behaviour of $G_\dag^\ell$, coherent with the growths in Figure \ref{fig:bounds_per_layers}, is shown in Figure \ref{fig:rates_per_layers}.
\begin{figure}[t!]
\centering 
\begin{subfigure}[t]{0.45\linewidth}
  \includegraphics[width=\columnwidth,keepaspectratio,clip]{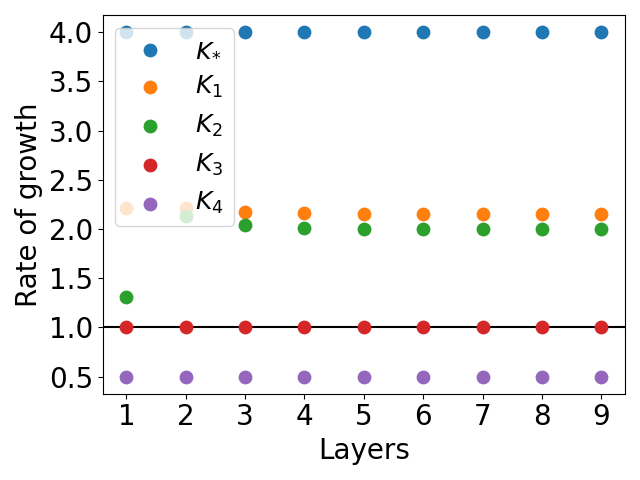} 
  \end{subfigure}
\begin{subfigure}[t]{0.45\linewidth}
  \includegraphics[width=\columnwidth,keepaspectratio,clip]{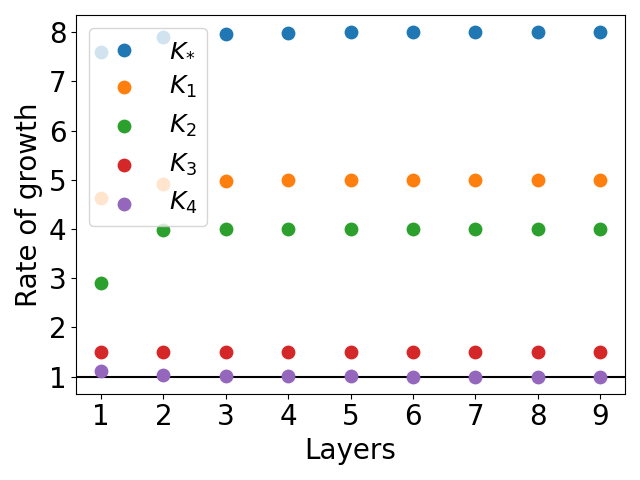} 
  \end{subfigure}
  \caption{Rates of growth for networks approximating the function $x^2$. The function $g$ is represented as in \eqref{eq:g_vers1} (left) or as in \eqref{eq:g_vers2} (right).}
  \label{fig:rates_per_layers}
\end{figure}

\subsection{Approximation of $xy$}\label{sec:xy}

We propose a way to construct a ReLU neural network to approximate a specific polynomial function of two variables.
This method is a completely different alternative to the polarization method of Yarostky for which we  refer to \cite{yarotsky2017error,desp:book}.
To the best of our knowledge, this construction is fully  original with respect to the literature.


Let us consider the two-dimensional domain $[-1,1]^2$ and the mesh $\cal{T}$ shown in Figure \ref{fig:takagi_2d_grid}. We denote by $\varphi_*$, $*\in\{\alpha,\beta,\gamma,\delta,A,B,\dots,H,I\}$ the $\P_1({\cal{T}})$ finite element basis function with value 1 at the node $*$ and 0 at the other ones. Such basis functions are used to construct the function $\Lambda:\R^2\rightarrow\R^2$, $\Lambda=(\Lambda_1, \Lambda_2)$, where
\begin{equation}\label{eq:lambda_def}
\begin{gathered}
\Lambda_1 = \varphi_\alpha - \varphi_\beta + \varphi_\gamma-\varphi_\delta\\
\Lambda_2 = \varphi_D - \varphi_A + \varphi_B - \varphi_C + \varphi_F - \varphi_E - \varphi_G + \varphi_H - \varphi_I,
\end{gathered}
\end{equation}
Denoting by $T$ the target function $T(x,y)=xy$, we define the function $e_0$ as $e_0=T-\frac14T\circ\Lambda$.

\begin{figure}[t!]
\centering 
  \includegraphics[width=0.4\columnwidth,keepaspectratio,clip]{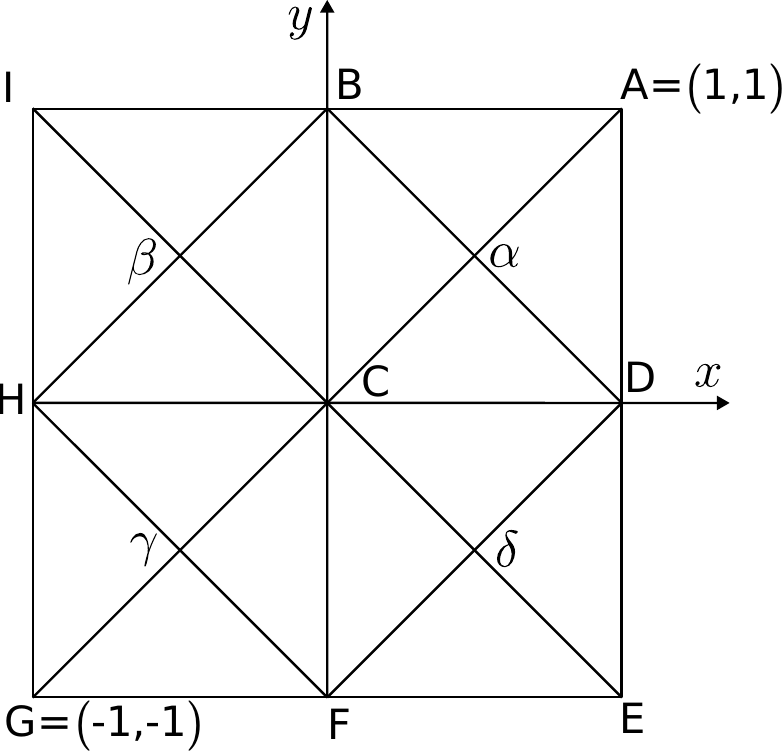} 
  \caption{Mesh used to construct the neural network approximating the function $xy$ on $[-1,1]^2$.}
  \label{fig:takagi_2d_grid}
\end{figure}

\begin{proposition}
The function $e_0$ is linear in each triangle of the mesh $\cal{T}$, i.e. $e_0\in\P_1({\cal{T}})$, and it interpolates $T$ in each node of the mesh $\cal{T}$.
\end{proposition}
\begin{proof}
Let us consider the triangle $T_{(\alpha,C,D)}$ of vertices $\alpha$, $C$ and $D$. Due to the local nature of the finite element basis functions, only the basis functions $\varphi_\alpha$, $\varphi_C$ and $\varphi_D$ are non-zero on $T_{(\alpha,C,D)}$. Therefore, $\Lambda$ can be computed as $\Lambda=(\varphi_\alpha, \varphi_D-\varphi_C)$. In particular, $\Lambda_{|T_{(\alpha,C,D)}} = (2y, 2x-1)$. A direct evaluation of the function $e_0$ in the triangle $T_{(\alpha,C,D)}$ shows that
\begin{equation*}
e_0(x,y) = T(x,y) - \frac14T\circ\Lambda(x,y)= xy - \frac14 T(2y, 2x-1) 
= xy - \frac{(2y)(2x-1)}{4}
= xy - \frac{4xy - 2y}{4}
= \frac{y}{2}.
\end{equation*}
Moreover, we observe that:
$
e_0(0,0) = 0 = T(0,0)$, $e_0\left(\frac12,\frac12\right) = \frac14 = T\left(\frac12,\frac12\right)$
and $e_0(1,0) = 0 = T(1,0)$, 
i.e. $e_0$ and $T$ coincide on the vertices of the triangle.
Repeating the same computation in each triangle, it can be shown that $e_0$ is linear in each triangle and $e_0=T$ in each node of $\cal{T}$.
\end{proof}
From the definition of $e_0$, we can express $T$ as
\begin{equation}\label{eq:xy_net_def}
\begin{aligned}
T &= e_0 + \frac14T\circ\Lambda 
= e_0 + \frac14\left(e_0 + \frac14T\circ\Lambda\right)\circ\Lambda \\
&= e_0 + \frac14 e_0\circ\Lambda + \frac{1}{4^2}T\circ\Lambda\circ\Lambda
= e_0 + \frac14 e_0\circ\Lambda + \frac{1}{4^2}\left(e_0 + \frac14T\circ\Lambda\right)\circ\Lambda\circ\Lambda\\
&= e_0 + \frac14 e_0\circ\Lambda + \frac{1}{4^2}e_0\circ\Lambda\circ\Lambda + \frac{1}{4^3}T\circ\Lambda\circ\Lambda\circ\Lambda
= \sum_{r=0}^\infty \frac{1}{4^r} e_0 \circ\underbrace{\Lambda\circ\Lambda\circ\dots\circ\Lambda\circ\Lambda}_{r{\rm{\text{ times}}}}
=\sum_{r=0}^\infty \frac{g_r}{4^r},
\end{aligned}
\end{equation}
where
\begin{equation}\label{eq:xy_gr_def}
g_r = e_0 \circ\underbrace{\Lambda\circ\Lambda\circ\dots\circ\Lambda\circ\Lambda}_{r{\rm{\text{ times}}}}.
\end{equation}
The derived expression is similar to \eqref{eq:takagi_series}, i.e. the target function is expressed as an infinite series of terms that are combinations of piecewise linear functions. Therefore, if one is able to exactly represent the functions $e_0$ and $\Lambda$ as ReLU neural networks, then it is possible to efficiently represent the target function $T$.

Let us consider the reference square $[-1,1]^2$ and let us subdivide it in four triangles by joining the vertices with the origin. Let $\hat\varphi$ be the piece-wise linear function with value 1 at the origin and 0 at the vertices and which is linear in each triangle. This is our reference basis function and can be constructed, for example, as:
\begin{equation}\label{eq:ref_phi_1}
\hat\varphi(x,y) = R\left[2-R(x+y)-R(x-y)-R(1-x)\right],
\end{equation}
or
\begin{equation}\label{eq:ref_phi_2}
\hat\varphi(x,y) = R\left[1-R\left(\frac x2+\frac y2\right)-R\left(\frac x2-\frac y2\right)-R\left(-\frac x2+\frac y2\right)-R\left(-\frac x2-\frac y2\right)\right].
\end{equation}
Such functions can be expressed as ReLU neural networks with architecture $\widetilde f_2\circ R\circ\widetilde f_1\circ R\circ\widetilde f_0$. We point out that adding the external operator $R$ (and thus the linear layer $\widetilde f_2$) is not necessary when considering $\hat\varphi$ on the reference square $[-1,1]^2$, but it is important in the neural network construction to obtain local basis functions. Using suitable linear change of variables $\cal{L}_*$, we construct the local basis function $\varphi_*$, $*\in\{\alpha,\beta,\gamma,\delta,B,C,D,F,H\}$ as $\varphi_*=\hat\varphi\circ{\cal{L}}_*$. Denoting by $\widetilde f_0^*$ the new linear layer defined as $\widetilde f_0^*=\widetilde f_0\circ {\cal{L}}_*$, each basis function $\phi_*$ can be exactly computed as $\widetilde f_2\circ R\circ\widetilde f_1\circ R\circ\widetilde f_0^*$. Lastly, we compute $\varphi_A$, $\varphi_E$, $\varphi_G$ and $\varphi_I$ as
\[
\varphi_A(x,y) = R(R(x+y-1)), \hspace{1cm} \varphi_E(x,y) = R(R(x-y+1)),
\]
\[
\varphi_G(x,y) = R(R(-x-y-1)), \hspace{1cm} \varphi_I(x,y) = R(R(-x+y+1)),
\]
where the double application of the ReLU operator is needed to obtain neural networks as deep as the previous ones.

Let $f_\varphi = f_2\circ R\circ f_1\circ R\circ f_0$ be the ReLU neural network with 2 inputs and 13 outputs obtained combining the neural networks of the 13 basis functions $\varphi_*$, $*\in\{\alpha,\beta,\gamma,\delta,A,B,\dots,H,I\}$. Let $f_3:\R^{13}\rightarrow\R^2$ be the linear operator mapping the vector $[\varphi_\alpha,\varphi_\beta,\varphi_\gamma,\varphi_\delta,\varphi_A,\varphi_B,\dots,\varphi_H,\varphi_I]$ to the vector $[\Lambda_1,\Lambda_2]$ according to \eqref{eq:lambda_def} and let $f_{3,2}$ be the linear operator $f_3\circ f_2$. A ReLU neural network exactly representing the operator $\Lambda$ is thus $f_\Lambda = f_{3,2}\circ R\circ f_1\circ R\circ f_0$. Analogously, let $f_4:\R^{13}\rightarrow\R$ be the linear operator mapping the vector $[\varphi_\alpha,\varphi_\beta,\varphi_\gamma,\varphi_\delta,\varphi_A,\varphi_B,\dots,\varphi_H,\varphi_I]$ to the function $e_0$ and let $f_{4,2}$ be the linear operator $f_4\circ f_2$. Then, a ReLU neural network exactly representing the function $e_0$ is $f_{e_0} = f_{4,2}\circ R\circ f_1\circ R\circ f_0$.

Substituting the derived neural networks in equation \eqref{eq:xy_gr_def}, $g_r$ can be thus exactly represented as:
\begin{equation}
\begin{aligned}
g_r &= e_0 \circ\underbrace{\Lambda\circ\Lambda\circ\dots\circ\Lambda\circ\Lambda}_{r{\rm{\text{ times}}}} 
= f_{e_0} \circ\underbrace{f_\Lambda\circ f_\Lambda\circ\dots\circ f_\Lambda\circ f_\Lambda}_{r{\rm{\text{ times}}}} \\
&= \left(f_{4,2}\circ R\circ f_1\circ R\circ f_0\right) \circ \underbrace{\left(f_{3,2}\circ R\circ f_1\circ R\circ f_0\right)\circ \dots \circ \left(f_{3,2}\circ R\circ f_1\circ R\circ f_0\right)}_{r{\rm{\text{ times}}}}.
\end{aligned}
\end{equation}
As done previously, we merge the consecutive linear layers $f_0$ and $f_{3,2}$ into a new linear layer $f_{0,3,2}$. We denote the resulting ReLU neural network by $f_{g_r}$. Finally, as in Section \ref{sec:x2}, the target function $T$ is approximated by truncating the series \eqref{eq:xy_net_def} and by enlarging the resulting neural network to store and sum all the intermediate terms involved in such a series.
Adding a single neuron to each layer in Section \ref{sec:x2} was enough because the target function and the approximation were always positive. Instead, since in this new case they can be both positive and negative in the domain, we need to add two neurons to propagate the information by representing the identity operator $I$ as $I(x) = R(x) - R(-x)$.

We are now able to test the bounds discussed in Section \ref{sec:fcnn} on this new benchmark neural network. In Figure \ref{fig:bounds_per_layers_xy}, such bounds are evaluated on neural networks of different length and constructed using the representations \eqref{eq:ref_phi_1} and \eqref{eq:ref_phi_2}. We point out that, when \eqref{eq:ref_phi_1} is used, the resulting neural network is a series of alternating layers of width 33 and 15, whereas when \eqref{eq:ref_phi_2} is used, the widths of the layers are 42 and 15. In both cases, it is possible to observe that all bounds grow exponentially with respect to the number of layers and that, as expect from the theory, $K_4 \le K_1 \le K_*$ and $K_4 \le K_3 \le K_*$. Note that it is not possible to evaluate the $K_2$ bound because the layers are already too large.

In Figure \ref{fig:rates_per_layers_xy} we show the rates of growths of the bounds shown in Figure \ref{fig:bounds_per_layers_xy}. Coherently with the exponential growth of such bounds, all the rates of growths are positive and constant. We also highlight that the bounds and the rates of growth are smaller when \eqref{eq:ref_phi_2} is used, even if it leads to a larger network. The improved performance is probably due to the symmetry of such a representation. Indeed, in Section \ref{sec:x2}, we observed that the bounds are sharper and the rates of growth smaller when a symmetric representation of the function $g$ is chosen.

\begin{figure}[t!]
\centering 
\begin{subfigure}[t]{0.45\linewidth}
  \includegraphics[width=\columnwidth,keepaspectratio,clip]{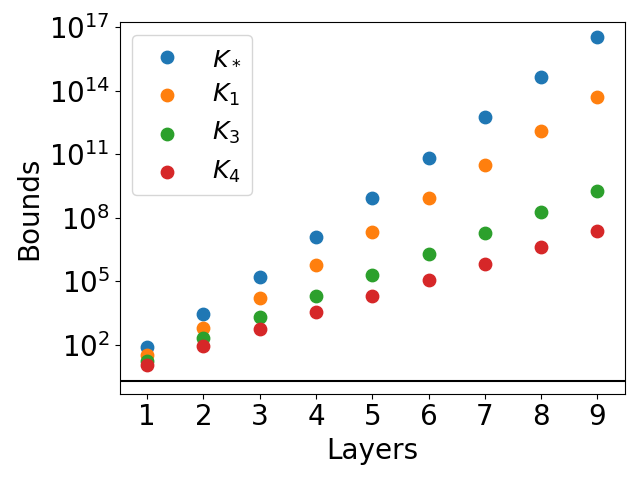} 
  \end{subfigure}
\begin{subfigure}[t]{0.45\linewidth}
  \includegraphics[width=\columnwidth,keepaspectratio,clip]{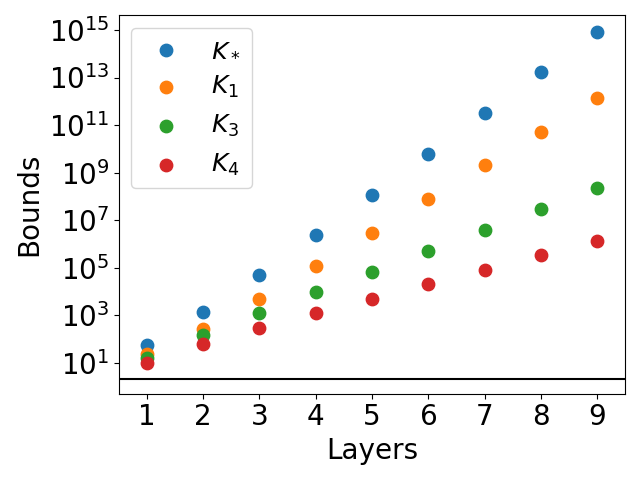} 
  \end{subfigure}
  \caption{Lipschitz bounds for networks approximating the function $xy$. The function $\hat\varphi$ is represented as in \eqref{eq:ref_phi_1} (left) or as in \eqref{eq:ref_phi_2} (right).}
  \label{fig:bounds_per_layers_xy}
\end{figure}

\begin{figure}[t!]
\centering 
\begin{subfigure}[t]{0.45\linewidth}
  \includegraphics[width=\columnwidth,keepaspectratio,clip]{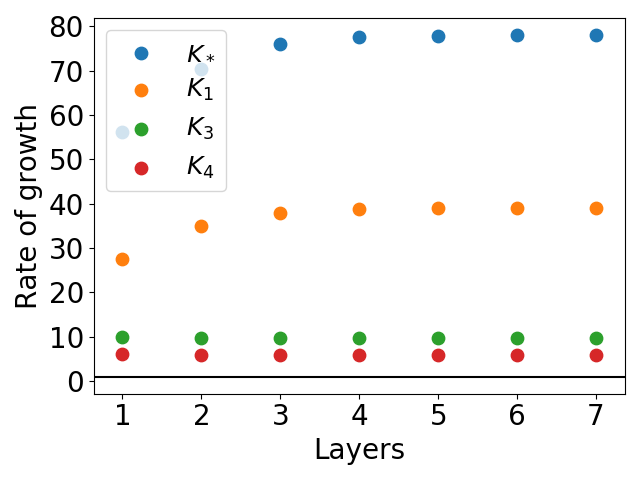} 
  \end{subfigure}
\begin{subfigure}[t]{0.45\linewidth}
  \includegraphics[width=\columnwidth,keepaspectratio,clip]{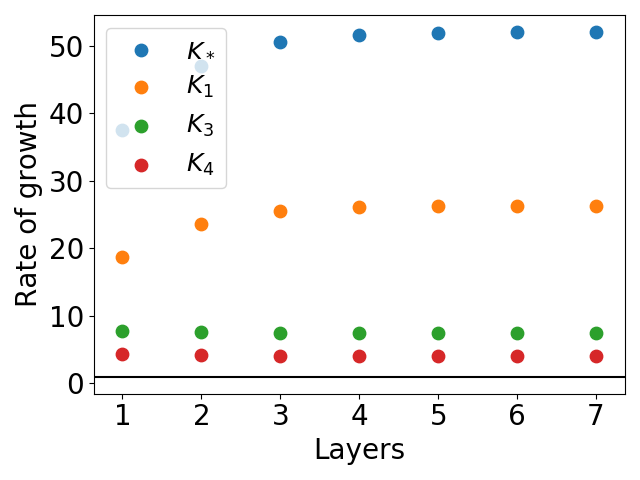} 
  \end{subfigure}
  \caption{Rates of growth for networks approximating the function $xy$. The function $\hat\varphi$ is represented as in \eqref{eq:ref_phi_1} (left) or as in \eqref{eq:ref_phi_2} (right).}
  \label{fig:rates_per_layers_xy}
\end{figure}

\subsection{Convolutional neural networks}
In this section, we test the bounds proposed in Section \ref{sec:conv} on convolutional neural networks trained on the MNIST dataset. To compare the performance of the discussed bounds we train three different convolutional neural networks with the following architectures:
\begin{itemize}
\item Model A: Convolutional layer (5 filters) - Average-pooling layer - Convolutional layer (10 filters) - Max-pooling layer - Dense layer (20 neurons) - Dense layer (10 neurons);
\item Model B: Convolutional layer (5 filters) - Max-pooling layer - Convolutional layer (10 filters) - Max-pooling layer - Dense layer (20 neurons) - Dense layer (10 neurons);
\item Model C: Convolutional layer (10 filters) - Max-pooling layer - Convolutional layer (20 filters) - Max-pooling layer - Dense layer (40 neurons) - Dense layer (10 neurons).
\end{itemize}
Note that from Model A to Model B we only change the first pooling layer from an \textit{Average-pooling layer} to a \textit{Max-pooling layer}, whereas from Model B to Model C we only double the number of filters and neurons in the inner layers.

We adopt the ReLU activation function everywhere, apart from the last layer where we use the softmax activation function. We always use a glorot normal initialization and we train the networks with the ADAM optimizer \cite{kingma2014adam}, with standard values of the hyperparameters, for 100 epochs. Each network is trained three times with different $l^2$ regularization parameters $\lambda_{\rm{reg}}$ and we evaluate all the available bounds. The numerical results are shown in Tables \ref{tab:conv_res_no_reg}, \ref{tab:conv_res_reg1e-3} and \ref{tab:conv_res_reg1e-2} for $\lambda_{\rm{reg}}=0$, $\lambda_{\rm{reg}}=1e-3$ and $\lambda_{\rm{reg}}=1e-2$ respectively.

Coherently with the theory presented in Section \ref{sec:conv}, the inequalities $K_4 \le K_1 \le K_*$ and $K_4 \le K_3 \le K_*$ hold for both the implicit and the explicit approach. The numerical results show that, in the considered cases, the implicit approach bounds are all sharper than the ones obtained with the explicit approach, both in the $l^1$ norm and in the $l^\infty$ norm. For example, note that the value of $K_4$ computed with the explicit approach is always greater than the value of $K_*$ compute with the implicit one. Lastly, we highlight that higher values of $\lambda_{\rm{reg}}$ can be used to obtain neural networks with lower Lipschitz bounds. 

\begin{table}[htbp]
\footnotesize
 \caption{Lipschitz bounds on convolutional neural networks trained without regularization. The accuracy on the test set is always between 97\% and 98\%.}
\begin{center}
\begin{tabular}{|c|c||cccc||cccc|}
\hline
\multirow{2}{*}{Model} & \multirow{2}{*}{Approach} & \multicolumn{4}{c||}{$l^1$ \vphantom{$M^{M^{M^a}}$}}                                                                   & \multicolumn{4}{c|}{$l^\infty$}                                                                               \\ \cline{3-10} 
                       &                           & \multicolumn{1}{c|}{$K_*$ \vphantom{$M^{M^{M^a}}$}} & \multicolumn{1}{c|}{$K_1$} & \multicolumn{1}{c|}{$K_3$} & $K_4$ & \multicolumn{1}{c|}{$K_*$}   & \multicolumn{1}{c|}{$K_1$}   & \multicolumn{1}{c|}{$K_3$}   & $K_4$            \\ \hline\hline
\multirow{2}{*}{Model A}               & Explicit                  &\multicolumn{1}{c|}{1.777e6} & \multicolumn{1}{c|}{1.347e5} & \multicolumn{1}{c|}{2.362e5} & {3.119e4} & \multicolumn{1}{c|}{2.836e7} & \multicolumn{1}{c|}{2.381e6} & \multicolumn{1}{c|}{8.320e6} & \multicolumn{1}{c|}{9.236e5}      \\ 
              & Implicit                  & \multicolumn{1}{c|}{2.172e3} & \multicolumn{1}{c|}{4.061e2} & \multicolumn{1}{c|}{8.016e2} & \textbf{2.349e2} & \multicolumn{1}{c|}{7.602e4} & \multicolumn{1}{c|}{1.268e4} & \multicolumn{1}{c|}{3.364e4} & \multicolumn{1}{c|}{\textbf{8.197e3}}   \\ \hline

\multirow{2}{*}{Model B}               & Explicit                  &\multicolumn{1}{c|}{3.567e8} & \multicolumn{1}{c|}{9.815e6} & \multicolumn{1}{c|}{7.780e6} & {6.364e5} & \multicolumn{1}{c|}{9.811e8} & \multicolumn{1}{c|}{6.063e7} & \multicolumn{1}{c|}{2.584e8} & \multicolumn{1}{c|}{1.752e7}     \\ 
                & Implicit                  & \multicolumn{1}{c|}{8.916e3} & \multicolumn{1}{c|}{1.759e3} & \multicolumn{1}{c|}{3.190e3} & \textbf{9.946e2} & \multicolumn{1}{c|}{3.925e5} & \multicolumn{1}{c|}{6.421e4} & \multicolumn{1}{c|}{1.286e5} & \multicolumn{1}{c|}{\textbf{3.328e4}}   \\ \hline

\multirow{2}{*}{Model C}                & Explicit                 & \multicolumn{1}{c|}{8.630e9} & \multicolumn{1}{c|}{2.179e8} & \multicolumn{1}{c|}{1.381e8} & {1.176e7} & \multicolumn{1}{c|}{2.555e10} & \multicolumn{1}{c|}{1.478e9} & \multicolumn{1}{c|}{4.704e9} & \multicolumn{1}{c|}{2.861e8}         \\ 
                & Implicit                  &\multicolumn{1}{c|}{1.348e4} & \multicolumn{1}{c|}{2.488e3} & \multicolumn{1}{c|}{3.790e3} & \textbf{1.216e3} & \multicolumn{1}{c|}{6.388e5} & \multicolumn{1}{c|}{1.007e5} & \multicolumn{1}{c|}{1.602e5} & \multicolumn{1}{c|}{\textbf{3.832e4}}  \\ \hline

\end{tabular}
\end{center}
 \label{tab:conv_res_no_reg}
\end{table}

\begin{table}[htbp]
\footnotesize
 \caption{Lipschitz bounds on convolutional neural networks trained with $l^2$ regularization with parameter $\lambda_{\rm{reg}}=1e-3$. The accuracy on the test set is always between 97\% and 98\%.}
\begin{center}
\begin{tabular}{|c|c||cccc||cccc|}
\hline
\multirow{2}{*}{Model} & \multirow{2}{*}{Approach} & \multicolumn{4}{c||}{$l^1$ \vphantom{$M^{M^{M^a}}$}}                                                                   & \multicolumn{4}{c|}{$l^\infty$}                                                                               \\ \cline{3-10} 
                       &                           & \multicolumn{1}{c|}{$K_*$ \vphantom{$M^{M^{M^a}}$}} & \multicolumn{1}{c|}{$K_1$} & \multicolumn{1}{c|}{$K_3$} & $K_4$ & \multicolumn{1}{c|}{$K_*$}   & \multicolumn{1}{c|}{$K_1$}   & \multicolumn{1}{c|}{$K_3$}   & $K_4$            \\ \hline\hline
\multirow{2}{*}{Model A}               & Explicit                  &\multicolumn{1}{c|}{4.753e5} & \multicolumn{1}{c|}{4.102e4} & \multicolumn{1}{c|}{5.699e4} & {9.077e3} & \multicolumn{1}{c|}{9.618e6} & \multicolumn{1}{c|}{8.795e5} & \multicolumn{1}{c|}{1.902e6} & \multicolumn{1}{c|}{2.333e5}       \\ 
              & Implicit                  & \multicolumn{1}{c|}{7.542e2} & \multicolumn{1}{c|}{1.638e2} & \multicolumn{1}{c|}{2.552e2} & \textbf{9.340e1} & \multicolumn{1}{c|}{2.428e4} & \multicolumn{1}{c|}{4.569e3} & \multicolumn{1}{c|}{9.621e3} & \multicolumn{1}{c|}{\textbf{2.777e3}}   \\ \hline

\multirow{2}{*}{Model B}               & Explicit                  & \multicolumn{1}{c|}{1.074e8} & \multicolumn{1}{c|}{3.996e6} & \multicolumn{1}{c|}{1.996e6} & {2.608e5} & \multicolumn{1}{c|}{2.102e8} & \multicolumn{1}{c|}{2.177e7} & \multicolumn{1}{c|}{6.191e7} & \multicolumn{1}{c|}{6.672e6}      \\ 
                & Implicit                  & \multicolumn{1}{c|}{2.686e3} & \multicolumn{1}{c|}{6.127e2} & \multicolumn{1}{c|}{9.377e2} & \textbf{3.614e2} & \multicolumn{1}{c|}{8.409e4} & \multicolumn{1}{c|}{1.651e4} & \multicolumn{1}{c|}{3.592e4} & \multicolumn{1}{c|}{\textbf{1.089e4}}  \\ \hline

\multirow{2}{*}{Model C}                & Explicit                 & \multicolumn{1}{c|}{2.839e9} & \multicolumn{1}{c|}{7.436e7} & \multicolumn{1}{c|}{3.108e7} & {3.226e6} & \multicolumn{1}{c|}{5.999e9} & \multicolumn{1}{c|}{3.607e8} & \multicolumn{1}{c|}{1.007e9} & \multicolumn{1}{c|}{6.863e7}         \\ 
                & Implicit                  & \multicolumn{1}{c|}{4.436e3} & \multicolumn{1}{c|}{9.379e2} & \multicolumn{1}{c|}{1.014e3} & \textbf{4.224e2} & \multicolumn{1}{c|}{1.500e5} & \multicolumn{1}{c|}{2.482e4} & \multicolumn{1}{c|}{4.233e4} & \multicolumn{1}{c|}{\textbf{1.131e4}}   \\ \hline

\end{tabular}
\end{center}
 \label{tab:conv_res_reg1e-3}
\end{table}

\begin{table}[htbp]
\footnotesize
 \caption{Lipschitz bounds on convolutional neural networks trained with $l^2$ regularization with parameter $\lambda_{\rm{reg}}=1e-2$. The accuracy on the test set is always between 95\% and 97\%.}
\begin{center}
\begin{tabular}{|c|c||cccc||cccc|}
\hline
\multirow{2}{*}{Model} & \multirow{2}{*}{Approach} & \multicolumn{4}{c||}{$l^1$ \vphantom{$M^{M^{M^a}}$}}                                                                   & \multicolumn{4}{c|}{$l^\infty$}                                                                               \\ \cline{3-10} 
                       &                           & \multicolumn{1}{c|}{$K_*$ \vphantom{$M^{M^{M^a}}$}} & \multicolumn{1}{c|}{$K_1$} & \multicolumn{1}{c|}{$K_3$} & $K_4$ & \multicolumn{1}{c|}{$K_*$}   & \multicolumn{1}{c|}{$K_1$}   & \multicolumn{1}{c|}{$K_3$}   & $K_4$            \\ \hline\hline
\multirow{2}{*}{Model A}               & Explicit                  & \multicolumn{1}{c|}{1.058e5} & \multicolumn{1}{c|}{1.107e4} & \multicolumn{1}{c|}{9.377e3} & {2.222e3} & \multicolumn{1}{c|}{1.565e6} & \multicolumn{1}{c|}{1.734e5} & \multicolumn{1}{c|}{2.993e5} & \multicolumn{1}{c|}{4.756e4}       \\ 
              & Implicit                  & \multicolumn{1}{c|}{1.562e2} & \multicolumn{1}{c|}{4.721e1} & \multicolumn{1}{c|}{5.191e1} & \textbf{2.739e1} & \multicolumn{1}{c|}{3.678e3} & \multicolumn{1}{c|}{9.021e2} & \multicolumn{1}{c|}{1.859e3} & \multicolumn{1}{c|}{\textbf{6.424e2}}   \\ \hline

\multirow{2}{*}{Model B}               & Explicit                  &\multicolumn{1}{c|}{1.712e7} & \multicolumn{1}{c|}{7.353e5} & \multicolumn{1}{c|}{1.599e5} & {3.626e4} & \multicolumn{1}{c|}{2.878e7} & \multicolumn{1}{c|}{3.171e6} & \multicolumn{1}{c|}{4.590e6} & \multicolumn{1}{c|}{6.804e5}        \\ 
                & Implicit                  &\multicolumn{1}{c|}{4.280e2} & \multicolumn{1}{c|}{1.423e2} & \multicolumn{1}{c|}{1.772e2} & \textbf{1.016e2} & \multicolumn{1}{c|}{1.151e4} & \multicolumn{1}{c|}{2.823e3} & \multicolumn{1}{c|}{6.203e3} & \multicolumn{1}{c|}{\textbf{2.294e3}}  \\ \hline

\multirow{2}{*}{Model C}                & Explicit                 & \multicolumn{1}{c|}{4.950e8} & \multicolumn{1}{c|}{1.910e7} & \multicolumn{1}{c|}{3.120e6} & {7.327e5} & \multicolumn{1}{c|}{8.241e8} & \multicolumn{1}{c|}{7.749e7} & \multicolumn{1}{c|}{9.068e7} & \multicolumn{1}{c|}{1.170e7}         \\ 
                & Implicit                  & \multicolumn{1}{c|}{7.735e2} & \multicolumn{1}{c|}{2.337e2} & \multicolumn{1}{c|}{1.853e2} & \textbf{1.230e2} & \multicolumn{1}{c|}{2.060e4} & \multicolumn{1}{c|}{4.570e3} & \multicolumn{1}{c|}{6.574e3} & \multicolumn{1}{c|}{\textbf{2.307e3}}  \\ \hline

\end{tabular}
\end{center}
 \label{tab:conv_res_reg1e-2}
\end{table}

\section{Conclusion}\label{sec:conclusion}
In this work we analysed and compared various upper bounds of the Lipschitz constant of deep neural networks. We considered five different bounds: the naive one ($K_*$), the ones proposed in \cite{combettes2020lipschitz} and in \cite{virmaux2019lipschitz} ($K_1$ and $K_2$, respectively), and two novel bounds ($K_3$ and $K_4$). For fully-connected feed-forward neural networks, $K_*$, $K_1$ and $K_2$ are upper bounds in any $l^p$ matrix norm. However, we mainly focused on the $l^1$ and $l^\infty$ norms to provide new bounds and prove that they are sharper than the existing ones with the same computational cost, i.e. $K_3\le K_*$ and $K_4\le K_1$.

Two types of generalizations of such bounds to convolutional neural networks involving convolutional layers, linear layers, max-pooling layers and average-pooling layers are presented. The theoretical inequalities obtained for fully-connected feed-forward neural networks are also extended to prove that the generalization of $K_4$ is still the sharper bound among the available ones.

We provide numerical results on fully-connected feed-forward neural networks, with random weights or approximating specific polynomials, and on convolutional neural networks with different architectures. We also proposed two ways to construct neural networks converging to the functions $x\rightarrow x^2$ and $(x,y)\rightarrow xy$ exponentially with respect to the number of layers. This is important in order to improve the theoretical understanding of deep neural networks in a simplified scenarios, where multiple layers are present but the represented function is still known in closed form.

Future extension of this works include a theoretical analysis in other norms, the development of training strategies involving the proposed bounds to obtain more stable neural networks, and the generalization of the networks converging to $x\rightarrow x^2$ and $(x,y)\rightarrow xy$ to more general functions. 

\section*{Acknowledgements}
This work is funded by PEPER/IA.

\bibliographystyle{siam}
\bibliography{bibliography}
\end{document}